\documentclass[11pt]{article}
\usepackage[utf8]{inputenc}

\usepackage{xspace}
\usepackage{bbm}
\usepackage{amsfonts,amsmath,amssymb,amsthm, pbox}
\usepackage[colorlinks,citecolor=blue,bookmarks=true]{hyperref}

\usepackage{multirow}
\usepackage{dsfont} %

\usepackage{algorithmic,algorithm}

\usepackage[shortlabels]{enumitem}
\setitemize{noitemsep,topsep=0pt,parsep=0pt,partopsep=0pt}
\setenumerate{noitemsep,topsep=0pt,parsep=0pt,partopsep=0pt}

\makeatletter
\@ifundefined{theorem}{%
  \theoremstyle{definition}
  \newtheorem{definition}{Definition}
  \theoremstyle{plain}
  \newtheorem{theorem}{Theorem}
  \newtheorem{corollary}{Corollary}
  \newtheorem{lemma}{Lemma}

  \theoremstyle{remark}
  \newtheorem{remark}{Remark}

}{}
\makeatother
\newenvironment{proofof}[1]{\noindent{\bf Proof of {#1}:~~}}{\(\qed\)}

\newcommand{\ignore}[1]{}

\newcommand{\EE}{\mathbb{E}}

\newcommand{\NN}{\mathbb{N}}

\newcommand{\RR}{\mathbb{R}}

\newcommand{\expectation}[1]{\EE\left[#1\right]}

\newcommand{\reals}{\RR}

\def \cA     {{\cal A}}

\def \cG     {{\cal G}}

\def \cM     {{\cal M}}

\def \cT     {{\cal T}}

\def \cZ     {{\cal Z}}

\newcommand{\eg}{\textit{e.g.,}\xspace}
\newcommand{\ie}{\textit{i.e.,}\xspace}  %
\newcommand{\iid}{\textit{i.i.d.}} %

\def \ceil#1{{\lceil{#1}\rceil}}

\newcommand{\absv}[1]{\left|#1\right|}

\def \paren#1{{({#1})}}
\def \Paren#1{{\left({#1}\right)}}

\newcommand{\eqdef}{{:=}}

\newcommand{\probof}[1]{\Pr\Paren{#1}}

\def\ignore#1{}

\newcommand{\bi}{\begin{itemize}}
\newcommand{\ei}{\end{itemize}}

\def\orpro{\mathop{\mathchoice
   {\vee\kern-.49em\raise.7ex\hbox{$\cdot$}\kern.4em}
   {\vee\kern-.45em\raise.63ex\hbox{$\cdot$}\kern.2em}
   {\vee\kern-.4em\raise.3ex\hbox{$\cdot$}\kern.1em}
   {\vee\kern-.35em\raise2.2ex\hbox{$\cdot$}\kern.1em}}\limits}

\def\andpro{\mathop{\mathchoice
 {\wedge\kern-.46em\lower.69ex\hbox{$\cdot$}\kern.3em}
 {\wedge\kern-.46em\lower.58ex\hbox{$\cdot$}\kern.25em}
 {\wedge\kern-.38em\lower.5ex\hbox{$\cdot$}\kern.1em}
 {\wedge\kern-.3em\lower.5ex\hbox{$\cdot$}\kern.1em}}\limits}

\def\simge{\mathrel{%
   \rlap{\raise 0.511ex \hbox{$>$}}{\lower 0.511ex \hbox{$\sim$}}}}

\def\simle{\mathrel{
   \rlap{\raise 0.511ex \hbox{$<$}}{\lower 0.511ex \hbox{$\sim$}}}}

\newcommand{\prob}{\mathbb{P}}

\newcommand{\ns}{n}

\newcommand{\p}{p}

\newcommand{\eps}{\varepsilon}

\def\withcolors{0}
\def\withnotes{0}
\def\crwithcolors{0}
\def\withcolorsnew{0}

\ifnum\withnotes=1
\newcommand{\hl}[1]{{›\textcolor{red}{#1}}}
\newcommand{\zs}[1]{{\noindent \textit{\small\textcolor{brown}{ziteng: #1}}}}
\newcommand{\ts}[1]{{\noindent \textit{\small\textcolor{red}{theertha: #1}}}}
\newcommand \travis [1]{\textcolor{blue}{Travis: #1}}
\else
\newcommand{\hl}[1]{}
\newcommand{\zs}[1]{}
\newcommand{\ts}[1]{}
\newcommand{\travis}[1]{}
\fi
\ifnum\withcolors=1
\newcommand{\new}[1]{{\color{purple}{#1}}}
\else
\newcommand{\new}[1]{{{#1}}}
\fi

\ifnum\crwithcolors=1
\newcommand{\znew}[1]{{\color{red}{#1}}}
\else
\newcommand{\znew}[1]{{{#1}}}
\fi
\ifnum\withcolorsnew=1
\newcommand{\newer}[1]{{\color{purple}{#1}}}
\else
\newcommand{\newer}[1]{{{#1}}}
\fi

\usepackage{natbib}
\usepackage{graphicx}
\usepackage{fullpage}
\usepackage{algorithm, algorithmic}
\usepackage{pseudo}
\usepackage{bbm}
\usepackage{thm-restate}
\usepackage[capitalise]{cleveref}
\usepackage{hhline}

\newcommand \clip {\operatorname{clip}}
\newcommand \lap {\operatorname{Lap}}

\newcommand{\indic}[1]{\mathbbm{1}\left\{ #1\right\}}

\newcommand \ranks {\operatorname{ranks}}
\newcommand \rankerror {\operatorname{err}_{\text{rank}}}
\newcommand \ind {\mathbb{I}}
\newcommand{\dar}{d}

\title{Subset-Based Instance Optimality in Private Estimation}
\author{Travis Dick \and Alex Kulesza \and Ziteng Sun \and Ananda Theertha Suresh}
\date{
\texttt{\{tdick, kulesza, zitengsun, theertha\}@google.com} \\[2ex]
Google Research, New York}

\begin{document}

\maketitle
\begin{abstract}
  We propose a new definition of instance optimality for differentially private estimation algorithms. Our definition requires an optimal algorithm to compete, simultaneously for every dataset $D$, with the best private benchmark algorithm that (a) knows $D$ in advance and (b) is evaluated by its worst-case performance on large subsets of $D$. That is, the benchmark algorithm need not perform well when potentially extreme points are \emph{added} to $D$; it only has to handle the removal of a small number of real data points that already exist. This makes our benchmark significantly stronger than those proposed in prior work. We nevertheless show, \new{for real-valued datasets,} how to construct private algorithms that achieve our notion of instance optimality when estimating a broad class of dataset properties, including means, quantiles, and $\ell_p$-norm minimizers. For means in particular, we provide a detailed analysis and show that our algorithm simultaneously matches or exceeds the asymptotic performance of existing algorithms under a range of distributional assumptions.
\end{abstract}
\zs{Should we mention that our result focuses on the one-dimensional case more explicitly?}
\ts{yes.}
\ts{In the paper, we present means result first and then general instance optimality results. The abstract does it the other way round, should we be consistent?}
\zs{I think the abstract reads a bit nicer this way. }

\section{Introduction}

Differentially private algorithms intentionally inject noise to obscure the contributions of individual data points \citep{dwork2006calibrating,dwork2014algorithmic}. This noise, of course, reduces the accuracy of the result, so it is natural to ask whether we can derive a private algorithm that minimzes the cost to accuracy, ``optimally'' estimating some property $\theta(D)$ of a dataset $D$ given the constraint of differential privacy.

But what do we mean by optimal? Optimal worst-case loss is often achievable by algorithms that add noise calibrated to the global sensitivity of $\theta$ \citep{dwork2006calibrating, alda2017optimality,balle2018improving,fernandes2021laplace}. However, realistic datasets may have local sensitivities that are much smaller, making this a weak notion of optimality that does not reward reducing loss on typical ``easy'' examples \citep{nissim2007smooth,bun2019average}. At the other extreme, we might hope for instance optimality, where a single algorithm competes, simultaneously for all datasets $D$, with the best possible benchmark algorithm chosen with knowledge of $D$. It is easy to see that this is too strong: a constant algorithm that always returns exactly $\theta(D)$, regardless of its input, is perfectly private and gives zero loss on $D$. Of course, we cannot hope to achieve zero loss for all datasets at once.

Thus, there has been recent interest in finding variations of instance optimality that are strong and yet still achievable. \citet{asi2020near} gave a variant defined using local minimax risk, which softens the definition above by allowing an optimal algorithm's performance to degrade on $D$ whenever there is a second dataset $D'$ such that no private algorithm can simultaneously achieve low loss on both $D$ and $D'$. \citet{huang2021instance} gave a different variant of instance optimality for mean estimation in which the worst-case loss across nearby datasets sharing support with $D$ defines the risk for $D$. Related ideas were also explored by \citet{blasiok2019towards,brunel2020propose, mcmillan2022instance, dong2022nearly}, and others.

In this work we propose a new definition of instance optimality. We refer to our definition as \emph{subset optimality} because it calibrates the risk of a dataset $D$ using only subsets of $D$.

To motivate this approach, consider the basic semantics of differential privacy. A DP algorithm is required to give similar output distributions on $D$ and $D'$ whenever $D'$ is a \emph{neighbor} of $D$---that is, whenever $D'$ can be obtained from $D$ by adding or removing a single point. Technically, this definition is symmetric, but in practice addition and removal can have very different implications. For typical real datasets where data points tend to be similar to one another, removing a point may change the target property $\theta(D)$ only modestly, and a correspondingly modest amount of noise might ensure sufficiently similar output distributions. On the other hand, an \emph{added} point could potentially be an extreme outlier that dramatically changes $\theta(D)$, requiring a large amount of noise to conceal its presence.

Concretely, imagine we have a database reporting the annual incomes for $n$ households in a particular neighborhood, the largest of which happens to be \$100,000. We wish to privately compute the mean household income. Removing one of the existing households from the database can change the mean by at most roughly \$100,000$/n$. Adding a \emph{new} household, on the other hand, could have a dramatically larger effect---a new household's income might theoretically be, say, \$100,000,000, requiring 1000x more noise. Thus, an algorithm that is required to perform well only on subsets of the real dataset intuitively has an advantage over one that must perform well everywhere.

The surprising implication of our results is that this is not always true. We demonstrate how to construct subset-optimal differentially private algorithms that simultaneously compete, for all datasets $D$, with the best private algorithm that (a) knows $D$ in advance and (b) is evaluated by its worst-case performance over only (large) subsets of $D$. Subset optimality is achievable for a class of monotone properties that include means, quantiles, and other common estimators, despite being stronger than prior definitions in the literature.

We begin by describing our setting in more detail, defining subset optimality, and comparing our definition to those found in related work. We then show how to achieve subset optimality for mean estimation, giving an algorithm that simultaneously matches or exceeds the asymptotic performance of existing algorithms under a range of distributional assumptions (see \Cref{tab:results}). Finally, we
generalize the result for means and show how to construct optimal algorithms for a broad class of monotone properties.

\section{Problem formulation}

A dataset $D$ is a collection of points from the domain $[-R, R]$. Let $|D|$ be the cardinality of the set $D$. For two datasets $D$ and $D'$ we define the distance between them to be the number of points that need to be added to and/or removed from $D$ to obtain $D'$. More formally,
\[
\dar(D, D') \triangleq |D \setminus D'| + |D' \setminus D|.
\]
For example, $d(\{1, 2, 3\}, \{2, 3, 4\}) = 2$. We refer to $D$ and $D'$ as \emph{neighboring} datasets if $d(D, D') = 1$ \citep[Chapter 2]{dwork2014algorithmic}.

Note that this notion of neighboring datasets is sometimes called the \emph{add-remove} model, in contrast to the \emph{swap} model where all datasets are the same size and datasets are neighbors if they differ in a single point. We use the add-remove model here since we study algorithms that accept input datasets of various sizes. Because the add-remove model leads to stronger privacy than the swap model, our algorithms also provide guarantees in the swap model, with at most a factor of two increase in the privacy parameter $\eps$.

\ts{It might be good to get consistent numbering across arXiv and ICML versions. Can we get it as Definition $1.1$? Not sure how to do it?}
\zs{I tried a bit, but I couldn't find a way to do it.}
\begin{definition}[Differential privacy]
A randomized algorithm $A$ with  range $\mathcal{R}$ satisfies $\eps$-differential privacy if for any two neighboring datasets $D, D'$ and for any output $\mathcal{S} \subseteq \mathcal{R}$, it holds that 
 \[
 \text{Pr}[A(D) \in \mathcal{S}] \leq e^{\eps} \text{Pr}[A(D') \in \mathcal{S}].
 \]
\end{definition}
Let $\cA_{\eps}$ be the set of all $\eps$-DP algorithms. Our goal is to estimate some property of $D$, denoted by $\theta(D)$, and we use a loss function $\ell: R \times R \to R_{\geq0}$ to measure the performance of an algorithm $A$ on a dataset $D$ as
\[
\ell(A(D), \theta(D)).
\]

Given $\ell$, we would like to identify an algorithm $A$ that performs well not just on average or for certain datasets, but simultaneously for every dataset $D$. For each $D$, we will aim to compete with a benchmark algorithm that can be \emph{selected} using knowledge of $D$ but is \emph{evaluated} according to its performance on large subsets of $D$. In particular, we adopt the following definition of subset-based risk.
\begin{equation}\label{eqn:subset_risk}
    R(D,\eps) \triangleq \inf_{A \in \cA_\eps} \sup_{\substack{D' \subseteq D \\ |D'| \geq |D| - 1/\eps}} \EE\left[\ell\Paren{A(D'), \theta(D')}\right]\footnote{\new{We focus on the case when $\eps \le 1$ in this paper. To simply our notation, we often treat $1/\eps$ as a positive integer. If this is not the case, we can set $\eps' = 1/\ceil{1/\eps} \in (\eps/2, \eps]$. This will affect our results by at most constant factors.}}
\end{equation}
\zs{I moved a later footnote to here. Hope it is OK.}
The benchmark algorithm for $D$ is the one with the lowest worst-case loss on datasets obtained by removing up to $1/\eps$ elements from $D$. Intuitively, if it is feasible to perform well on all of these subsets at once, then the benchmark risk is small and an optimal algorithm will be expected to perform correspondingly well, even if \emph{adding} elements to $D$ would dramatically increase the benchmark algorithm's loss. When no private algorithm performs well on all large subsets of $D$, then an optimal algorithm will also be permitted to have larger loss.

\begin{definition} \label{def:subset}
We say an algorithm $A$ is \textit{subset-optimal} with respect to $\cA_\eps$ if there exist constants $\alpha$, $\beta$, and $c$ such that, for all $D$, we have 
\[
    \EE\left[\ell\Paren{A(D), \theta(D)}\right] \le \alpha \cdot  R(D,\eps) + \beta,
\]
and $A$ is $c \cdot \eps$-DP.
\end{definition}

Ideally, we want the constants $\alpha$ and $c$ to be close to 1 and $\beta$ to be close to $0$.
\znew{\begin{remark}
The constraint $|D'| \ge |D| - 1/\eps$ in \cref{eqn:subset_risk} could be generalized to $|D| - \tau$ for any $\tau \ge 0$.
Intuitively, smaller $\tau$
 would lead to a stronger notion of optimality, since the benchmark algorithm would need to perform well on fewer datasets. Our choice of $\tau = 1/\eps$
 gives the strongest possible optimality definition that remains achievable: for any $\tau \ll 1/\eps$, it is impossible to compete with the benchmark algorithm. To see this, consider the task of real mean estimation. Let $D_1$ contain $(\ns - 1/\eps)$ copies of $0$ and $1/\eps$ copies of $1$, and let $D_2$
 contain $(\ns - 1/\eps)$
 copies of $0$ and $1/\eps$ copies of $-1$. Since the add/remove distance between 
 $D_1$ and $D_2$ and 
 is $2/\eps$, standard packing arguments (e.g., \cref{lem:lower} or Lemma 8.4 in \citet{dwork2014algorithmic}) show that an 
$(\eps, \delta)$-DP algorithm cannot give significantly different answers on $D_1$ and $D_2$, and therefore must incur an error of at least $\Theta(1/n\eps)$ on one of them. On the other hand, a benchmark algorithm that knows the dataset can always output $1/\ns\eps$
 for $D_1$, and on subsets of $D_1$
 with size at least $|D_1| - \tau$, the error will be at most $\tau/\ns$ (and similarly for $D_2$). Thus $\tau \ll 1/\eps$
 does not yield an achievable definition of instance-optimality.
\end{remark}}

\textbf{Notation.} For a dataset $D = \{x_1\leq x_2 \leq\ldots \leq  x_\ns\}$, let $L_m(D)$ denote the multiset $\{x_1, x_2,\ldots, x_m\}$ (the $m$ lower elements of $D$) and $U_{m}(D)$ denote the multiset $\{x_{n-m+1}, \ldots, x_n\}$ (the $m$ upper elements of $D$).

\zs{Will a bigger preliminary/setting section help? Currently all the definitions are spread in different sections. }
\zs{Do you think a contribution section summarizing the results we have will be helpful here?}
\ts{I think that would be helpful!}

\new{
\subsection{Our Contributions.}

\paragraph{Subset-based instance optimality.}  We propose a new notion of instance-optimality (\cref{def:subset}) for private estimation. The notion has the advantage of only considering the effect of removing values from the dataset, which leads to tighter (or as tight) rates compared to other instance-optimal formulations that need to handle extreme data points. See \cref{sec:related} for a detailed discussion. Moreover, we propose \cref{alg:priv_est} based on private threshold estimation and the inverse sensitivity mechanism of \cite{asi2020instance}. For real-valued datasets, the algorithm is \emph{subset-optimal} for a wide range of monotone properties with arbitrary $\beta > 0$ and $\alpha$ and $c$ at most logarithmic in problem-specific parameters (see \cref{sec:monotone} and \cref{thm:general} for the precise definition and statement). 

\paragraph{Improvement on mean estimation.} For the task of private mean estimation (\cref{sec:mean}), we propose an efficient algorithm  (\cref{alg:SubsetOptimalMean}) that is \emph{subset-optimal}. In the statistical setting (\cref{sec:statistical_mean}) we show how this algorithm obtains distribution-specific rate guarantees that depend on all centralized absolute moments (\cref{coro:moment}). To the best of our knowledge, the rate improves upon the previously best-known distribution-specific results for distributions whose $k$th-moment is much smaller than its best sub-Gaussian proxy for some $k \ge 2$. 
For distribution families with concentration assumptions, the distribution-specific rate recovers (up to logarithmic factors) the min-max optimal rate for each corresponding family. Moreover, our proposed algorithm achieves this rate without explicit assumptions on which family the distribution comes from.
See \cref{tab:results} for a detailed comparison.
}
\begin{table*}[ht]
\centering
\caption{Results for statistical mean estimation.  $R(A, p)$ is the expected absolute error of $A$ given $\ns$ \iid~samples from $p$ (\cref{eqn:sme}). $M_k(p)$ denotes the $k$th absolute central moment of $p$ (\cref{def:moment}).  $\sigma_{\cG}(p)$ denotes the best sub-Gaussian proxy of $p$.
$\dagger$ is due to \cite{huang2021instance}.
$\ddagger$ is due to \citet{kamath2020private} and $^*$ is due to \citet{karwa2017finite}.}

\renewcommand{\arraystretch}{2}
\begin{tabular}{|c|c|c|c|}
\hline
 Assumption & Metric & Prior work & This work \\ \hline
   N.A. & $R(A, p)$ & $\! \tilde{O}\Paren{ \frac{\sqrt{M_2(p)}}{\sqrt{\ns}}\! + \! \frac{\sigma_{\cG}(p)}{\ns \eps} } \dagger$& $\! \tilde{O}\Paren{ \frac{\sqrt{M_2(p)}}{\sqrt{\ns}} \! +\! \min_{k}\frac{M_k(p)^{1/k}}{(n\eps)^{1-1/k}}} \!$\footnotemark\\[6pt]  \hline
      Bounded Moment   & $\max_{p: M_k(p) \le m_k} R(A, p)$ & $O\Paren{\frac{m_k^{1/k}}{\sqrt{\ns}} \! +\! \frac{m_k^{1/k}}{(n\eps)^{1-1/k}}} \ddagger$ & $\tilde{O}\Paren{\frac{m_k^{1/k}}{\sqrt{\ns}} \! +\! \frac{m_k^{1/k}}{(n\eps)^{1-1/k}}}$   \\[6pt]  \hline
   Sub-Gaussian &$\max_{p :\sigma_\cG(p) \le \sigma} R(A, p)$ & $\tilde{O}\Paren{\frac{\sigma}{\sqrt{\ns}} + \frac{\sigma}{\ns\eps}}^*$   & $\tilde{O}\Paren{\frac{\sigma}{\sqrt{\ns}} + \frac{\sigma}{\ns\eps}}$ \\[6pt]
    \hline
\end{tabular}
\label{tab:results}
\end{table*}

\subsection{Related Work} \label{sec:related}

\paragraph{\new{Instance-optimality in private estimation.}} Several variations of instance optimality for differential privacy have been studied recently.

\citet{asi2020near} initiated the study of instance optimality using the following notion of \textit{local minimax risk}:
\[
R_{1} (D, \eps) = \sup_{D'} \inf_{A \in A_{\eps}} \sup_{\tilde{D} \in \{D, D'\}}
\EE[\ell(A(\tilde{D}), \theta(\tilde{D}))].
\]
They showed that the inverse sensitivity mechanism gives nearly optimal results with respect to this notion of risk. However, for mean estimation in one dimension, with all values bounded in $[-R, R]$, it can be shown that
\[
R_{1} (D, \eps) \gtrapprox \frac{R}{n\eps}
\]
for every dataset $D$.

\footnotetext{\newer{Similar result is also obtained in the independent work of \cite{dong2023universal}.}}
In contrast, subset optimality provides much tighter guarantees for mean estimation, roughly replacing the full range $R$ with the range actually spanned by $D$, as described in the sections below. For general losses, \citet{asi2020near} showed that
\begin{equation}\label{eqn:asi}
    R_{1} (D, \eps) \approx \max_{D': d(D, D') \leq 1/\eps} \ell(\theta(D), \theta(D')),
\end{equation}
while we show that
\[
R (D, \eps) \approx \max_{D': d(D, D') \leq 1/\eps, D' \subseteq D} \ell(\theta(D), \theta(D')),
\]
which is strictly tighter due to the subset constraint. (As illustrated above, the difference can be dramatic for realistic $D$.) \znew{\citet{mcmillan2022instance} studied a quantity similar to $R_1(D, \eps)$ in the setting where $D$ is drawn from a distribution. We focus on the comparison to \citet{asi2020instance} since it is most relevant to our setting.}

\newcommand{\support}{{\rm supp}}
\citet{huang2021instance} considered a slightly different notion of instance optimality given by $R_2(D,\eps) =$
\begin{equation*}
\inf_{A \in \cA_\eps}\! \sup_{\substack{\support(D') \subseteq \support(D) \\ d(D, D') = 1}} \! \! \inf_{\eta} \left\{\Pr\left(\ell\Paren{A(D'),\! \! \theta(D')}\! \!> \! \!\eta \right) \! \!<\! \! \frac{2}{3} \right\},
\end{equation*}
where $\support(D)$ denotes the set of unique elements %
in the dataset $D$.
They proposed an algorithm for $d$-dimensional mean estimation and showed that it is $O(\sqrt{d/\rho})$-optimal for $\rho$-zCDP (concentrated differential privacy). Their definition and results differ from $R(D, \eps)$ and $R_1(D, \eps)$ in two basic ways. First, their definition is a high probability definition, whereas the others are in expectation. Second, even for one-dimensional mean estimation, their proposed algorithm is only $O(1/\sqrt{\rho})$ competitive with the lower bound, and they further show that no algorithm can achieve a better competitive guarantee.  Our definition (modulo expectation) coincides with this definition for $\eps = 1$; however, we are able to construct constant optimal algorithms for general $\eps$. We also show that our definition of optimality is achievable for target properties beyond just means.

\znew{
\begin{remark} \label{rmk:huang}
The definition of \citet{huang2021instance} can be extended by changing $d(D, D') = 1$ to $d(D, D') \le 1/\eps$, bringing it closer to our definition. However, this leads to an instance-dependent risk of \[\tilde{R}_2(D, \eps) \approx \sup_{\substack{\support(D') \subseteq \support(D)\\ d(D, D') \le 1/\eps}}\ell(\theta(D), \theta(D')),\] which can be much larger than bound in \cref{thm:MeanLowerBound}. 
\end{remark}
}

\znew{In \cref{app:lp}, we use $\ell_p$ minimization as a concrete example to compare these instance-dependent risks and show that our new definition gives significant quantitative improvements for specific datasets.}

It is worth remarking here that both \citet{asi2020instance} and \citet{huang2021instance} provide achievability results when the dataset is supported over a high-dimensional space while our result mainly focuses on one-dimensional datasets. Whether our subset-based instance optimality can be achieved in the high-dimensional setting is an interesting future direction to explore.

\newer{The notion of subset-based instance optimality is related to the notion of \textit{down sensitivity} \citep{raskhodnikova2016lipschitz, cummings2020down, dong2022r2t}, which measures the change of a property of a dataset with respect to removing points from the dataset. \cite{fang2022shifted} proposed estimators 
whose error scales with the down sensitivity for properties that are monotone with respect to adding points such as dataset max or distinct count. Our results apply to properties that are monotone in terms of stochastic dominance of datasets, which include mean, quantiles, and $\ell_p$-norm minimizers. The two monotonicity definitions do not imply each other in general and hence the results are directly comparable.}

\paragraph{Private statistical mean estimation.} Private mean estimation has been widely studied in the statistical setting, where the dataset is assumed to be generated \iid~from an underlying distribution. Classic methods such as the Laplace or Gaussian mechanism~\cite{dwork2014algorithmic} incur a privacy cost that scales with the worst-case sensitivity. However, recently, by assuming certain concentration properties of the underlying distribution (e.g., sub-Gaussianity \citep{karwa2017finite, cai2019cost, smith2011privacy, kamath2019privately, bun2019average, biswas2020coinpress}, bounded moments \citep{feldman2018calibrating, kamath2020private, hopkins2022efficient} and high probability concentration \citep{levy2021learning, huang2021instance}), it has been shown that the privacy cost can be improved to scale with the concentration radius of the underlying distribution. These algorithms are in some cases known to be (nearly) optimal for distributions satisfying these specific concentration properties. 

It is worth remarking here that most of these works consider high-dimensional distributions while our work only focuses on real-valued datasets. Moreover, for moment-bounded distributions, \cite{kamath2020private} achieves constant-optimal estimation risk while our obtained rate is only optimal up to logarithmic factors. The results are outlined in Table~\ref{tab:results}.  
\newer{Independently, \citet{dong2023universal} obtained similar results in \cref{tab:results} for general distributions in terms of their moments. Our bounds are obtained by the instance-dependent bound in \cref{thm:stat_mean}, which might be of independent interest.}

Our mean estimation algorithm is similar to that of \cite{huang2021instance}. However, we choose a tighter threshold in the clipping bound estimation step, which is crucial in the analysis to achieve the instance-optimal bound.

\ignore{
\begin{table*}[ht]
\centering
\caption{Comparison of mean estimation results.  Rates marked by $\dagger$ are instance-dependent; the others are min-max rates over the assumed probability family. $M_k(p)$ denotes the $k$th absolute central moment of $p$ ($\infty$ if unbounded). All min-max rates are tight up to logarithmic factors.}
\renewcommand{\arraystretch}{2}
\begin{tabular}{|c|c|c|c|}
\hline
 \multicolumn{2}{c}{Setting} & This work & Prior work\\ \hline
  \multicolumn{2}{c}{\multirow{2}{*}{Distribution-specific}}&\multirow{2}{*}{
  $\! \tilde{O}\Paren{ \frac{\sqrt{M_2(p)}}{\sqrt{\ns}} \! +\! \min_{k}\frac{M_k(p)^{1/k}}{(n\eps)^{1-1/k}}}\dagger \!$
}
  &
  $\tilde{O}\Paren{\frac{\sqrt{M_2(p)}}{\sqrt{\ns}}  \! +\! \frac{M_\infty(p)}{\ns\eps}}   \dagger $
\citep{huang2021instance} \\ & & & $O\Paren{\frac{R}{\ns\eps}}\dagger$ \citep{asi2020instance}  \\[6pt] 
  \hline 
  \multirow{3}{*}{Min-max}
     & $k$th moment &  $\tilde{O}\Paren{\frac{M_k(p)^{1/k}}{\sqrt{\ns}} \! +\! \frac{M_k(p)^{1/k}}{(n\eps)^{1-1/k}}}$ &  $O\Paren{\frac{M_k(p)^{1/k}}{\sqrt{\ns}} \! +\! \frac{M_k(p)^{1/k}}{(n\eps)^{1 -1/k}}}$ \citep{kamath2020private} \\[6pt]  \cline{2-4}
   & $\sigma$-Sub-Gaussian & $\tilde{O}\Paren{\frac{\sigma}{\sqrt{\ns}} + \frac{\sigma}{\ns\eps}}$    &  $\tilde{O}\Paren{\frac{\sigma}{\sqrt{\ns}} + \frac{\sigma}{\ns\eps}}$ \citep{karwa2017finite}\\[6pt]
   \cline{2-4}
    & $B$-Bounded range &  $\tilde{O}\Paren{\frac{B}{\sqrt{\ns}} + \frac{B}{\ns\eps}}$ & $\tilde{O}\Paren{\frac{B}{\sqrt{\ns}} +\frac{B}{\ns\eps}}$ \citep{huang2021instance} \\[6pt] 
    \hline
\end{tabular}
\label{tab:results}
\end{table*}
}

\section{Subset-Optimal Private Means}  \label{sec:mean}
We start by presenting \new{an efficient} subset-optimal algorithm for estimating means; in Section~\ref{sec:monotone} we will generalize this approach to a larger class of properties.

Let $\mu(D)$ denote the mean of a dataset $D$. We will use $\ell(x,y) = |x-y|$ as our loss function. \new{Our main result is stated in the theorem below.

\begin{theorem}\label{thm:SubsetOptimalMeans}
    Let $D \subset [-R,R]$ be a multiset of points and let $\hat \mu$ be the output of \Cref{alg:SubsetOptimalMean} with parameters $R$ and $\eps > 0$.
    Publishing $\hat \mu$ is $3\eps$-differentially private and for any $\gamma > 0$, we have
        \[
    \expectation{|\hat \mu - \mu(D)|}
    = O\left(
        R\left(D, \eps\right) \ln \frac{R\eps}{\gamma} + \frac{\gamma}{\eps}
    \right).
    \]
\end{theorem}
}

We begin by establishing an \znew{up-to-constant characterization of} the subset-based risk for mean estimation under this loss.

\begin{lemma}\label{thm:MeanLowerBound}
    For any $\eps \in (0,1]$
    and multiset $D \subset \reals$ with $|D| > \frac{1}{\eps}$,
    \begin{equation*}
    R(D,\eps)
    = \znew{c_{D, \eps} \cdot  \Paren{\mu(D\setminus L_{\frac{1}{\eps}}) - \mu(D\setminus U_{\frac{1}{\eps}})},}
    \end{equation*}
    \znew{where $c_{D, \eps} \in [1/(2e^2), 1]$.}
\end{lemma}
\znew{The upper bound is straightforward;} the lower bound proof is based on standard packing arguments~\citep{dwork2014algorithmic} and appears in Appendix~\ref{app:MeanLowerBound}.

Now we turn to designing a subset-optimal algorithm that is competitive with this lower bound for every dataset $D$ \new{and prove \cref{thm:SubsetOptimalMeans}.}

First, we describe a straightforward algorithm for private mean estimation of bounded datasets. Pseudocode is given in \Cref{alg:AddRemoveMean}, and privacy and utility analyses are given in \Cref{thm:AddRemoveMean}. We use $D + m$ to denote $\{x+m\ |\ x\in D\}$, and $\clip(D, [l,u])$ to denote $\{\min(\max(x, l), u)\ |\ x \in D\}$.

\begin{algorithm}
\noindent\textbf{Input:} Multiset $D \subset [l,u]$, $\eps > 0$.
\begin{pseudo}
    Let $w = u - l$ and $m = \frac{l+u}{2}$.\\
    Let $D' = D - m$.\\
    Let $\hat n = n + Z_n$, where $n=|D'|$, $Z_n \sim \lap(\frac{2}{\eps})$.\\
    Let $\hat s = s + Z_s$, where $s = \sum_{x\in D'}x$, $Z_s \sim \lap(\frac{w}{\eps})$.
    \\
    Output $\hat \mu = \clip(\frac{\hat s}{\hat n}, [-\frac{w}{2},\frac{w}{2}]) + m$.
\end{pseudo}
\caption{Bounded mean estimation}
\label{alg:AddRemoveMean}
\end{algorithm}

We provide the proof in Appendix~\ref{app:AddRemoveMean}.
\begin{lemma}\label{thm:AddRemoveMean}
    Let $[l,u]$ be any interval and $\eps > 0$ be any privacy parameter.
    Publishing $\hat \mu$ output by \Cref{alg:AddRemoveMean} is $\eps$-differentially private. Furthermore,
    \[
    \expectation{|\hat \mu - \mu(D)|} 
    \leq \frac{3(u-l)}{|D|\eps}.
    \]
\end{lemma}

In order to construct a subset-optimal mean algorithm from \Cref{alg:AddRemoveMean}, the high level idea is to first find an interval $[\hat l, \hat u]$ that contains all but a small number of outliers from the dataset $D$, clip $D$ to $[\hat l, \hat u]$, and finally apply \Cref{alg:AddRemoveMean}. The error of this algorithm will have two main components: error incurred by clipping the data to $[\hat l, \hat h]$, and error due to the noise added by \Cref{alg:AddRemoveMean}. Our analysis shows that both error components are not significantly larger than the lower bound given by \Cref{thm:MeanLowerBound}.

We start by describing the subroutine that we use to choose $\hat l$ and $\hat u$; following \Cref{thm:MeanLowerBound}, our goal will be to find $\hat l$ and $\hat u$ that delineate approximately $1/\eps$ elements of $D$ each.

\subsection{Private Thresholds} \label{sec:PrivateThresholds}

We present an $\eps$-differentially private algorithm that, given a multiset $D$ of real numbers and a target rank $r$, outputs a threshold $\tau \in \reals$ that is approximately a rank-$r$ threshold for $D$.

Roughly, $\tau$ being a rank-$r$ threshold for $D$ means that $r$ points in $D$ are less than or equal to $\tau$. However, if $D$ has repeated points then there can be ranks for which no such threshold exists. (In the extreme, consider the dataset $D = \{x, \dots, x\}$, containing $n$ copies of the same point; exactly $0$ or $n$ points are less than or equal to any threshold $\tau$.)

We will therefore define a rank-$r$ threshold in a slightly more general way so that there is at least one rank-$r$ threshold for every $r \in \{0, \dots, |D|\}$. This definition is consistent with the standard definition of quantiles for distributions with point-masses.
\begin{definition}\label{defn:rank}
    Let $D$ be any multiset of real numbers. 
    We say that $\tau \in \reals$ is a rank-$r$ threshold for $D$ if $\sum_{x \in D} \ind\{x < \tau\} \leq r$ and $\sum_{x \in D} \ind\{x \leq \tau\} \geq r$. 
    That is, there are at most $r$ points strictly smaller than $x$, and at least $r$ points that are greater than or equal to $x$.
    For any threshold $\tau$, let $\ranks(\tau, D)$ denote the set of ranks $r$ such that $\tau$ is a rank-$r$ threshold for $D$.
    See \Cref{fig:rankDefinition} for an example of this rank definition in a dataset with repeated points.
\end{definition}

\begin{figure}
    \centering
    \includegraphics[width=0.4\columnwidth]{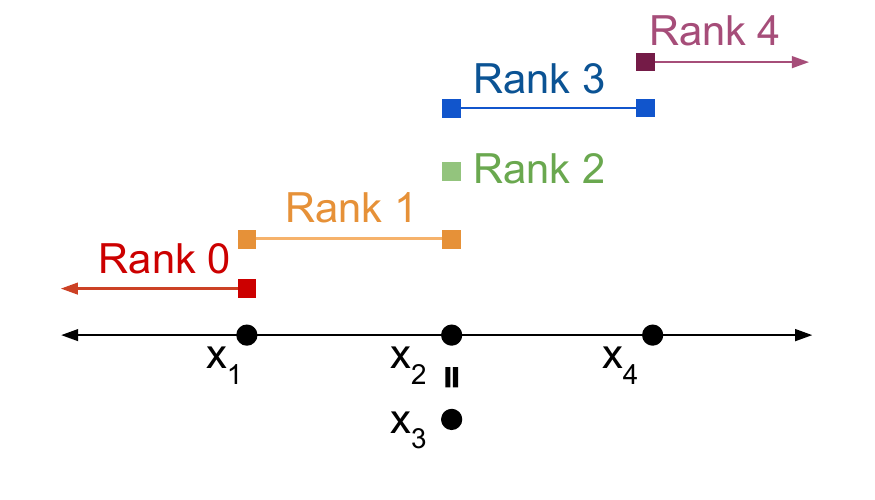}
    \caption{
    Example of ranks as defined in \Cref{defn:rank}. There are 4 points with $x_2 = x_3$.
    For each rank $r \in \{0, \dots, 4\}$ we show the interval of points that are rank-$r$ thresholds.
    For $r = 0, 1, \dots, 4$, the intervals of rank-$r$ thresholds are given by $I_0 = [-\infty, x_1]$, $I_1 = [x_1, x_2]$, $I_2 = \{x_2\}$, $I_3 = [x_3, x_4]$, and $I_4 = [x_4, \infty]$, respectively.
    }
    \label{fig:rankDefinition}
\end{figure}

\begin{definition}
    The \textit{rank error} of a threshold $\tau$ for dataset $D$ and target rank $r$ is
    \[
    \rankerror(\tau, r, D) = \min_{r' \in \ranks(\tau,D)} |r - r'|.
    \]
    The rank error measures how close $\tau$ is to being a rank-$r$ threshold for $D$ and is equal to 0 if and only if $\tau$ is a rank-$r$ threshold.
\end{definition}

When privately estimating rank-$r$ thresholds for a dataset $D$, we will incur a bicriteria error: our output will be close (on the real line) to a threshold with low rank error.
\begin{definition}
    Let $D$ be any multiset of real numbers. We say that $\tau$ is an $(\alpha,\beta)$-approximate rank-$r$ threshold for $D$ if there exists $\tau' \in \reals$ so that $|\tau - \tau'| \leq \alpha$ and $\rankerror(\tau', r, D) \leq \beta$.
\end{definition}

Our algorithm for finding approximate rank-$r$ thresholds is an instance of the exponential mechanism.
The loss (or negative utility) function that we minimize is parameterized by the distance error $\alpha > 0$, and the loss of a threshold $\tau$ is defined to be the minimum rank-error of any threshold $\tau'$ within distance $\alpha$ of $\tau$.

Intuitively, by allowing the loss function to ``search'' in a window around $\tau$, we guarantee that there is always an interval of width $\alpha$ with loss zero (namely, any interval centered around a true rank-$r$ threshold).
This is sufficient to argue that the exponential mechanism outputs a low-loss threshold with high probability and, by definition of the loss, this implies that there is a nearby threshold with low rank-error.
Pseudocode is given in \Cref{alg:PrivateThreshold}.
Since the density $f$ is piecewise constant with at most $2|D|$ discontinuities at locations $x \pm \alpha$ for $x \in D$, it is possible to sample from $f$ in time $O(|D| \log |D|)$ (see the proof of \Cref{thm:PrivateThreshold} for details).
\Cref{thm:PrivateThreshold}, which is proved in \Cref{sec:PrivateThresholdsAppendix}, shows that this algorithm outputs an $(\alpha,\beta)$-approximate threshold.

\begin{algorithm}[t]
\noindent\textbf{Input:} Dataset $D \subset [a,b]$, target rank $r$, data range $[a,b]$, distance $\alpha > 0$, privacy parameter $\eps > 0$.
\begin{pseudo}
    Define $\ell(\tau) = \min_{\tau - \alpha \leq \tau' \leq \tau+\alpha} \rankerror(\tau', r, D)$.\\
    Let $f(\tau) = \exp(-\frac{\eps}{2}\ell(\tau)) / \int_a^b \exp(-\frac{\eps}{2}\ell(\tau))\, d\tau$. \\
    Output sample $\hat \tau$ in $[a,b]$ drawn from density $f$.
\end{pseudo}
\caption{Threshold estimation}
\label{alg:PrivateThreshold}
\end{algorithm}

\begin{restatable}{theorem}{thmPrivateThreshold}\label{thm:PrivateThreshold}
    Fix data range $[a,b]$, $\eps > 0$, and distance error $0 < \alpha \leq \frac{b-a}{2}$. For any dataset $D \subset [a,b]$ and rank $1 \leq r \leq |D|$, let $\hat \tau$ be the output of \Cref{alg:PrivateThreshold} run on $D$ with parameters $[a,b]$, $\alpha$, and $\eps$.
    Publishing $\hat \tau$ satisfies $\eps$-DP and, for any $\zeta > 0$, with probability at least $1-\zeta$, $\hat \tau$ is an $(\alpha,\beta)$-approximate rank-$r$ threshold for $D$ with $\beta = \frac{2}{\eps} \ln \frac{b-a}{\alpha\zeta}$.
    Moreover, \Cref{alg:PrivateThreshold} can be implemented with $O(|D| \log |D|)$ running time. 
\end{restatable}

Next, we argue that when the dataset is supported on a grid $\cZ = \{z_1, \dots, z_m\}$, where $z_{i+1} = z_i + \gamma$, running \Cref{alg:PrivateThreshold} with a sufficiently small distance parameter $\alpha$ and rounding to the nearest grid point results in a $(0, \beta)$-approximate rank-$r$ threshold with high probability. The proof of \Cref{cor:threshold} is in \Cref{sec:PrivateThresholdsAppendix}.

\begin{restatable}{corollary}{corDiscreteThreshold}
\label{cor:threshold}
    Let $\cZ = \{z_1 \leq \dots \leq z_m\}$ be such that $z_{i+1} = z_i + \gamma$ for all $i = 1, \dots, m-1$ and let $D$ be any multiset supported on $\cZ$.
    Let $\hat \tau$ be the output of  \Cref{alg:PrivateThreshold} run on $D$ with parameters $[a,b] = [z_1, z_m]$, $\alpha = \gamma/3$, and $\eps > 0$, and let $\tilde \tau$ the closest point in $\cZ$ to $\hat \tau$. Then for any $\zeta > 0$, with probability at least $1-\zeta$ we have that $\tilde \tau$ is a $(0, \beta)$-approximate rank-$r$ quantile for $D$ with $\beta = \frac{2}{\eps} \ln \frac{3m}{\zeta}$.
\end{restatable}

\subsection{Mean Estimation}

We now present the pseudo-code for our subset-optimal mean estimation algorithm in \Cref{alg:SubsetOptimalMean} and give %
the \new{proof sketch of \Cref{thm:SubsetOptimalMeans}.}

\begin{algorithm*}
\textbf{Input:} Range $R$, dataset $D \subset [-R, R]$, privacy parameter $\eps' > 0$ and $\alpha > 0$.
\begin{pseudo}  
    Define $\alpha = \frac{\gamma}{|D|}$, $\zeta = \frac{\alpha}{R|D|\eps}$, $\beta = \frac{2}{\eps}\ln\frac{2R}{\alpha\zeta}$, $t_l = \frac{1}{\eps} + \beta$, and $t_u = |D| - \frac{1}{\eps} - \beta$.\footnotemark\\ 
    Let $\hat l$ be the output of \Cref{alg:PrivateThreshold} run on $D$ with parameters $[a,b] = [-R,R]$, $r=t_l$, $\alpha$, and $\eps$.\\
    Let $\hat u$ be the output of \Cref{alg:PrivateThreshold} run on $D$ with parameters $[a,b] = [-R,R]$, $r = t_u$, $\alpha$, and $\eps$.\\
    Let $\hat \mu$ be the output of \Cref{alg:AddRemoveMean} run on $D$ with interval $[a,b] = [\hat l, \hat u]$ and privacy parameter $\eps$.\\
    Output $\hat \mu$.
\end{pseudo}
\caption{Subset optimal mean estimation}
\label{alg:SubsetOptimalMean}
\end{algorithm*}
\footnotetext{
Note that when $\tau$ is a rank-$r$ threshold for a dataset $D$, $-\tau$ is a rank $|D|-r$ threshold for $-D$. Therefore, we can use \Cref{alg:PrivateThreshold} to find an approximate rank-$(|D|-r)$ threshold for $D$ by negating an approximate rank-$r$ threshold for $-D$.}

\begin{proof}[Proof sketch of \cref{alg:SubsetOptimalMean}] We provide the proof sketch here and provide a detailed proof in Appendix~\ref{app:SubsetOptimalMeans}.
    Our goal is to show that the expected error of \Cref{alg:SubsetOptimalMean} is not much larger than
    \[
    R(D,\eps) \geq \frac{\mu(D\setminus L_{\frac{1}{\eps}}) - \mu(D\setminus U_{\frac{1}{\eps}})}{2e^2}.
    \]

    With probability at least $1-\zeta$, by \Cref{thm:PrivateThreshold} we are guaranteed that $\hat l$ and $\hat u$ are $(\alpha,\beta)$-approximate rank-$t_l$ and rank-$t_u$ thresholds, respectively for the values of $\alpha$ and $\beta$ defined in \Cref{alg:SubsetOptimalMean}.
    In particular, this implies that there exist $l'$ and $t'_l$ such that $|\hat l - l'| \leq \alpha$, $|t'_l - t_l| \leq \beta$, and $l'$ is a rank-$t'_l$ threshold for $D$.
    Similarly, there exist $u'$ and $t'_u$ such that $|\hat u - u'| \leq \alpha$, $|t'_u - t_u| \leq \beta$, and $u'$ is a rank-$t'_u$ threshold for $D$.
    Let $G$ denote this high probability event.
    We first argue that conditioned on $G$, the expected loss of $\hat \mu$ is small (where the expectation is taken only over the randomness of \Cref{alg:AddRemoveMean}).
    To convert this bound into a bound that holds in expectation, we bound the error when $G$ does not hold by $2R$.

    Let $\hat \mu$ be the output of \Cref{alg:SubsetOptimalMean}.
    We decompose the error of $\hat \mu$ into three terms:
    \begin{align*}
        |\hat \mu - \mu(D)|
        &\leq   |\hat \mu - \mu(\clip(D, [\hat l, \hat u]))|\\
        &+ |\mu(\clip(D, [\hat l, \hat u]) - \mu(\clip(D, [l', u'])|\\
        &+ |\mu(\clip(D, [l', u'])) - \mu(D)|.
    \end{align*}
    Roughly speaking, the first term captures the variance incurred by using \Cref{alg:AddRemoveMean} to estimate the mean of the clipped data, the second term measures our bias due to $\alpha$ in our $(\alpha,\beta)$-approximate thresholds, and the third term measures the bias due to $\beta$.
    Our goal is to prove that all of these terms are not much larger than $R(D,\eps)$.
    
    \textit{Bounding first term.}
    At a high level, we argue that all points in $L_{\frac{1}{\eps}}$ are to the left of $\hat l + \alpha$, and all points in $R_{\frac{1}{\eps}}$ are to the right of $\hat u - \alpha$.
    It follows that the distance from any point in $L_{\frac{1}{\eps}}$ to any point in $U_{\frac{1}{\eps}}$ is at least $\hat u - \hat l - 2\alpha$.
    In particular, this guarantees that the difference between the means of $D\setminus L_{\frac{1}{\eps}}$ and $D\setminus U_{\frac{1}{\eps}}$ must be at least $\frac{\hat u - \hat l - 2\alpha}{\eps(|D|-\frac{1}{\eps})}$, since we move $1/\eps$ points a distance at least $\hat u - \hat l - 2\alpha$.
    This expression is close to the loss incurred by \Cref{alg:AddRemoveMean} when run on the clipped dataset.
    
    \textit{Bounding the second term.}
    The key idea behind bounding the second term is that, whenever $\hat l$ is close to $l'$ and $\hat u$ is close to $u'$, then clipping a point $x$ to $[\hat l, \hat u]$ is approximately the same as clipping it to $[l', u']$.

    \textit{Bounding the third term.}
    Our bound for the third term is the most involved.
    At a high level, we show that the bias introduced by clipping $D$ to the interval $[l', u']$ is at most the worst ``one-sided'' clipping bias incurred clipping the points to the left of $l'$ or to the right of $u'$.
    To see this, observe that when we clip from both sides, the left and right biases cancel out.
    Next, we argue that clipping points to the left of $l'$ (or to the right of $u'$) introduces less bias than \emph{removing} those points.
    This step bridges the gap between \Cref{alg:AddRemoveMean} which clips points and the lower bound on $R(D,\eps)$, which removes points.
    We argue that the number of points removed to the left of $l'$ or to the right of $u'$ is not much larger than $\frac{1}{\eps}$ and use \Cref{lem:RemoveExtraBound} to show that the resulting bias is not much larger than if we had removed exactly $\frac{1}{\eps}$ points instead.
    Finally, to finish the bound, combine our two ``one-sided'' bias bounds to show that the overall bias is never much larger than $R(D,\eps)$.
\end{proof}

\subsection{Intuition}

The lower bound in \Cref{thm:MeanLowerBound} is obtained by showing (roughly) that no private algorithm can reliably determine whether $O(1/\eps)$ outliers have been removed from $D$. In proving the upper bound in \Cref{thm:SubsetOptimalMeans}, then, the challenge is to show that those outliers can be identified and removed \emph{privately} without introducing asymptotically larger error even when $D$ is not known in advance.

This is possible in \Cref{alg:SubsetOptimalMean} due to a careful choice of the rank targets $t_l$ and $t_u$. In particular, if we are overly aggressive in trying to privately remove outliers, we run the risk of adding too much bias, since we are clipping away important information about the mean. On the other hand, if we are too tentative, we may end up with wide clipping thresholds that require adding too much variance (in the form of noise) when calling \Cref{alg:AddRemoveMean}. The key to our construction, therefore, is choosing rank targets such that the risk of excess bias and the risk of excess variance both exactly balance with the lower bound; that is, they match the error incurred by removing outliers in the first place.

There is no reason to think \textit{a priori} that this should be possible. Indeed, for certain properties (such as the mode), such a result does not seem to exist---errors due to over- or underestimating outliers can change the property arbitrarily. However, we show in the next section that the result for mean estimation can be extended to a relatively large class of common properties.

\section{Instance-optimal algorithm for monotone properties}

\begin{algorithm*}[t]
\caption{Subset-optimal monotone property estimation} \label{alg:priv_est}
\begin{algorithmic}[1]
\REQUIRE{Range $R$, dataset $D \subset [-R, R]$, privacy parameter $\eps > 0$, and discretization parameter $\beta$.\\
\textbf{Algorithms:} Private threshold algorithm \textbf{PrvThreshold} (\cref{alg:PrivateThreshold}).  Inverse sensitivity algorithm \textbf{InvSen}  \citep{asi2020instance} (\cref{alg:inverse_sensitivity}).
} 
\STATE Quantize each value in the dataset $D$ to the nearest multiple of $\beta$ and denote the quantized dataset by $D_\text{quant}$.
\STATE Set error probability $\eta=\frac{L\beta}{B}$, rank $r = \frac{32\log(6R/\eta\beta)}{\eps}$.
\STATE  $l = \textbf{PrvThreshold}( D_\text{quant}, r/4, [-R, R],\beta/3, \eps/4)$. 
\STATE $u = \textbf{PrvThreshold}( D_\text{quant}, |D| - r/4, [-R, R], \beta/3, \eps/4)$.
\STATE Let $D_\text{quant} = \{x_1\leq x_2 \leq, \ldots, \leq x_n\}$. For $i \leq 3r/2$, let $y_i = x_i - \beta$, for $i \geq n - 3r/2$ $y_i = x_i + \beta$ and otherwise $y_i = x_i$. Let $D'_{\text{quant}} = \{y_1\leq y_2 \leq, \ldots, \leq y_n\}$.
\STATE Prune the dataset to obtain
\[
    D_{[l, u]} = D'_{\text{quant}} \cap [l, u].
\]
\STATE Return the output of \textbf{InvSen} on $D_{[l, u]}$ with range $[l, u]$,  granularity $\beta$, and privacy parameter $\eps/2$.
\end{algorithmic}
\end{algorithm*}

\label{sec:monotone}
We now show that subset-optimal estimation algorithms can be constructed for any ``monotone'' property. We start by defining our notion of monotonicity.

\begin{definition}[First-order stochastic dominance \citep{
lehmann1955ordered,mann1947test}]
\label{def:first_order}
Let $D$ and $D'$ both be multisets of real numbers. $D'$ is said to \textit{dominate} $D$ (denoted $D' \succ D$) if, $\forall v \in \RR$,
\[
    \frac{\sum_{x' \in D'} \indic{x' \le v}}{|D'|} \leq  \frac{\sum_{x \in D} \indic{x \le v}}{|D|}.
\]
\end{definition}
In other words, first-order stochastic dominance requires the cumulative density function (CDF) of $D$ to be 
 larger than the CDF of $D'$ for all points on the real line. 

\begin{definition}[Monotone property]
A property is called \textit{monotone} if, for all $D', D$ with $D' \succ D$, we have
\[
    \theta(D') \ge \theta(D)
\]
or, for all $D', D$ with $D' \succ D$, we have
\[
    \theta(D') \le \theta(D).
\]
\end{definition}
Intuitively, the definition requires that if we move points from a dataset in one direction, we will always increase (or always decrease) the property. The family of monotone properties includes natural functions such as the mean, median, and quantiles. It also includes minimizers of $\ell_p$ distances, i.e.,
\[
\theta_p(D) = \arg\min_{y} \sum_{x \in D} |x-y|^p,
\]
and other common estimators. %

We also make the following assumptions on \new{the property $\theta$} and loss function $\ell$:
\begin{itemize}
\item \new{For any dataset $D$ supported on $[-R, R]$, $\theta(D) \in [-R, R]$\footnote{In general, we only need to assume the property is bounded and our result only depends on the bound logarithmically. We use the same $R$ here to simplify notations. }}.
\item $\ell$ is a metric; that is, it is commutative, it satisfies the triangle inequality, and $\ell(\theta, \theta) = 0$.
\item $\ell$ is finite and bounded for all datasets under consideration. Let $B = \sup_{D, D'} \ell(\theta(D), \theta(D'))$.
\item Whenever $\theta \geq \theta_1\geq \theta_2$,
\[
\ell(\theta, \theta_1) \leq \ell(\theta, \theta_2).
\]
\item $\ell$ is $L$-Lipschitz, as defined below\footnote{The constants in the two conditions below need not be the same. We use $L$ here for both to keep notations simple.}.
\new{\begin{itemize}
    \item Let $x_i(D)$ denote the $i^\text{th}$ largest element in $D$. For all 
$D, D'$ such that $|D| = |D'|$, we have
\[
    \ell(\theta(D), \theta(D')) \le L \max_{i \leq |D|} |x_i(D) - x_i(D')|.
\]
\item For all $\theta$ and $\theta_1 \neq \theta_2$, we have
\[
    \ell(\theta, \theta_1) \le  \ell(\theta, \theta_2) + L |\theta_1 - \theta_2|.
\]  
\end{itemize}
Observe that both mean and median are $1$-Lipschitz when $\ell(\theta, \theta') = |\theta -\theta'|$.
}
\end{itemize}

\new{Our main result is stated below.
\begin{theorem}
\label{thm:general} For any $\eps \in (0, c_\eps^{-1})$, there exists a $c_\eps \cdot \eps$-DP 
Algorithm (\cref{alg:priv_est}) with $c_\eps = 128\log(6RB/ L\beta^2)$, whose output $A(D)$ satisfies
\begin{align*}
 \EE[\ell(A(D),\theta(D))] 
 \leq 2e^2 R(D, \eps) + 7 L\beta.
\end{align*}
\end{theorem}

}
We first show a simple lower bound on $R(D,\eps)$ for general properties. This bound generalizes the mean estimation lower bound in \Cref{thm:MeanLowerBound}; the proof is given in Appendix~\ref{app:lower}.
\begin{lemma}
\label{lem:lower}
For $\eps \in [0, 1]$, let $S = \{(D_1, D_2): \min(|D_1|, |D_2|) \geq |D| - 1/\eps, 
d(D_1, D_2) \le 2/\eps\}$. If $S \neq \emptyset$, then $R(D,\eps)$ is at least
\[
 \frac{1}{2e^2} \cdot \max_{(D_1, D_2) \in S} \ell\Paren{\theta(D_1), \theta(D_2)}.
\]
\end{lemma}
We show that the above lower bound can be achieved (up to logarithmic factors).
The algorithm is given in \Cref{alg:priv_est}. It is similar in spirit to \Cref{alg:SubsetOptimalMean}, but we need to make a few modifications to ensure the algorithm works for general monotone properties. We briefly describe the steps in the algorithm below.

\textbf{Discretization}: As before, we will use the private threshold algorithm to remove outliers. The approximation guarantee in \Cref{thm:PrivateThreshold} has an additive rank error $\beta$ and an additive threshold error $\alpha$; however, for general properties, it is technically challenging to bound the effects of nonzero $\alpha$. To work around this, we first discretize the interval into steps of size $\beta$. This allows us to use  \Cref{cor:threshold} to get a guarantee with $\alpha = 0$.

\textbf{Private thresholds}: We then find the private thresholds $l$ and $u$ as in \Cref{alg:SubsetOptimalMean}. As noted above, these estimates come with $(0, \beta)$ approximation guarantees due to discretization.

\textbf{Pruning outliers}: In \Cref{alg:SubsetOptimalMean}, we clip outliers outside the thresholds. However, the effect of clipping is difficult to analyze generally. Instead, in \Cref{alg:priv_est} we simply prune outliers. Technically, it is possible for all values to lie on the thresholds, in which case we might not prune any elements. Hence, for ease of analysis, we deliberately move a small fraction of points outside the thresholds.

\textbf{Inverse sensitivity mechanism.} Finally, while in \Cref{alg:SubsetOptimalMean} we directly computed the private mean of the clipped dataset, here we use the inverse sensitivity mechanism \citep{asi2020instance} to estimate the desired property.

\section{Implications on private statistical mean estimation.} \label{sec:statistical_mean}

\new{In this section, we apply our mean estimation algorithm to the statistical mean estimation (SME) task where $D = X^\ns$, which are \iid~samples from a distribution $\p$ with mean $\mu$. And the performance of the algorithm is measured by the expected distance from the mean,
\begin{equation} \label{eqn:sme}
    R_{\rm SME}(A, p) \eqdef \EE\left[ \absv{A(D) - \mu}\right].
\end{equation}
We apply \cref{alg:SubsetOptimalMean} on $D$ and obtain obtain distribution-specific bounds on $R_{\rm SME}(A, p)$. 
For distribution families with various concentration assumptions,
we show that our instance-based bound is almost as tight (up to logarithmic factors) as algorithms designed for specific distribution families. We first state a generic result for statistical mean estimation.
}

\begin{theorem} \label{thm:stat_mean}
    Let 
    $D = X^\ns$ be \iid~samples from a distribution $p$ with mean $\mu$ and $A$ be \cref{alg:SubsetOptimalMean}. We have
    \begin{align*}
       R_{\rm SME}(A, p) \le \expectation{\absv{\mu(D) - \mu}} 
        + C \cdot \expectation{\absv{\mu(D  \setminus L_{\frac{1}{\eps}}) - \mu(D \setminus U_{\frac{1}{\epsilon}}) }},
    \end{align*}
    where 
    $C$ hides logarithmic factors in the problem parameters.
\end{theorem}
\new{
\begin{proof}
\begin{align*}
     R_{\rm SME}(A, p)  %
     = \EE\left[ \absv{A(D) - \mu}\right] 
     \le  \EE\left[ \absv{\mu(D) - \mu}\right]  + \EE\left[ \absv{A(D) - \mu(D)}\right].
\end{align*}
Applying \cref{thm:SubsetOptimalMeans} to the second term directly leads to the claim.
\end{proof}
}
\new{The bound in \cref{thm:stat_mean} can be hard to compute for a specific distribution.} For distributions with bounded moments, we can obtain explicit upper bounds on the quantities above. 
\begin{definition}\label{def:moment}
Let $p$ be a distribution supported on $\RR$ with mean $\mu$. Its $k$th absolute central moment is denoted as
\[
    M_k(p) \eqdef \EE_{X \sim p}\left[ |X - \mu(p)|^k\right]
\]
if it is finite; otherwise $M_k(p) = \infty$.
\end{definition}
In Appendix~\ref{app:cor_moment}, we prove the following result for statistical mean estimation on distributions with bounded moments.
\begin{corollary} \label{coro:moment}
    For any distribution $p$ over $[-R, R]$, \cref{alg:SubsetOptimalMean} satisfies
    \begin{align*}
       R_{\rm SME}(A, p)  = \tilde{O}\Paren{ \frac{M_2(p)^{1/2}}{\sqrt{\ns}} + \min_{k \ge 2}\frac{M_k(p)^{1/k}}{(n\eps)^{1-1/k}}}.
    \end{align*}
\end{corollary}

\new{Note that \cref{alg:SubsetOptimalMean} obtains the above instance-specific rate without any knowledge on the underlying distribution $p$. Moreover, for specific distribution families such as subgaussian distributions and distributions with bounded $k$th moments ($k \ge 2$), \cref{coro:moment} implies almost tight min-max rates.

\paragraph{Subgaussian distributions.} A distribution $p$ is subgaussian with proxy $\sigma$ if $\forall t \ge 0$,
\[
    \prob(|X - \mu| \ge t) \le 2 \exp \Paren{-\frac{t^2}{\sigma^2}}.
\]
We denote all such distributions as $\cG_\sigma$. For such distributions, we have
\[
    \max_{p \in \cG_\sigma} R_{\rm SME}(A, p) = \tilde{O}\Paren{\frac{\sigma}{\sqrt{\ns}} + \frac{\sigma}{\ns \eps}}.
\]
This matches the optimal rate for sub-Gaussian distributions (\eg, in \citet{karwa2017finite, kamath2019privately}).

\paragraph{Distributions with bounded moments.} Let $\cM_{k, m}$ be the family of distributions with $M_k(p) \le m$, we have
\[
    \max_{p \in \cM_{k, m}} R_{\rm SME}(A, p) = \tilde{O}\Paren{\frac{M_k^{1/k}}{\sqrt{\ns}} \! +\! \frac{M_k^{1/k}}{(n\eps)^{1-1/k}}}.
\]
This matches the optimal rate for distributions with bounded $k$th moment (\eg, in \citet{kamath2020private}). We list the detailed comparisons in \cref{tab:results}. 
}

\paragraph{Extending to higher dimensions.} For $(\eps, \delta)$-DP mean estimation in the high-dimensional case, \znew{algorithms in \citet{levy2021learning, huang2021instance} rely on pre-processing techniques} (\eg random rotation) and apply a one-dimensional estimation algorithm to each dimension. Our algorithm can also be combined with this procedure to obtain similar bounds \znew{since our algorithm provably provides an instance-optimal solution to each one-dimensional problem}. We leave exploring better instance-specific bounds in high dimension as a direction for future work. 

\section{Conclusion}

 We proposed a new definition of instance optimality for differentially private estimation and showed that our notion of instance optimality is stronger than those proposed in prior work. We furthermore constructed private algorithms that achieve our notion of instance optimality when estimating a broad class of monotone properties. We also showed that our algorithm matches the asymptotic performance of prior work under a range of distributional assumptions on dataset generation.

\bibliography{references}
\bibliographystyle{abbrvnat}
\newpage
\onecolumn
\appendix

\section{Proofs for \Cref{sec:PrivateThresholds}}\label{sec:PrivateThresholdsAppendix}

We will make use of the following characterization of rank-$r$ thresholds. 
\begin{lemma}
    Let $D$ be any multiset of real numbers. Then $\tau \in \reals$ is a rank-$r$ threshold for $D$ if and only if every point $x \in L_r$ satisfies $x \leq \tau$ and every point $x \in U_{|D|-r}$ satisfies $x \geq \tau$.
\end{lemma}
\begin{proof}
    First we show that if $\tau$ is a rank-$r$ threshold for $D$, then every point $x \in L_r$ satisfies $x \leq \tau$ and every point $x \in U_{|D| - r}$ satisfies $x \geq \tau$.
    By definition, we have that $\sum_{x \in D} \ind\{x \leq \tau\} \geq r$, which means that there are at least $r$ points in $D$ that are less than or equal to $\tau$.
    Since $L_r$ contains the $r$ smallest points in $D$, every point in $L_r$ must also be less than or equal to $\tau$.
    Similarly, by definition we have that $r \geq \sum_{x \in D} \ind\{x < \tau\} = |D| - \sum_{x \in D} \ind\{x \geq \tau\}$, which means that there are at least $|D|-r$ points in $D$ that are greater than or equal to $\tau$.
    Since $U_{|D|-r}$ contains the largest $|D|-r$ points in $D$, they must all be greater than or equal to $\tau$.
    This proves the first implication.
    
    Now suppose that $\tau$ is a threshold with the property that every $x \in L_r$ satisfies $x \leq \tau$ and every $x \in U_{|D| - r}$ satisfies $x \geq \tau$.
    We will show that this implies that $\tau$ is a rank-$r$ threshold.
    Let $x \in D$ be any point with $x < \tau$.
    We know that all such points must belong to $L_r$ (since $x \in U_{|D|-r}$ would violate our assumption).
    It follows that $\sum_{x \in D} \ind\{x < \tau\} \leq |L_r| = r$.
    Now let $x \in D$ be any point with $x > \tau$.
    We know that all such points must belong to $U_{|D|-r}$ (since $x \in L_r$ would violate our assumption).
    It follows that $\sum_{x \in D} \ind \{x \leq \tau\} = |D| - \sum_{x \in D} \ind\{x > \tau\} \geq |D| - |U_{|D| - r}| = r$.
    Together, these arguments show that $\sum_{x \in D} \ind\{x < \tau\} \leq r$ and $\sum_{x \in D} \ind\{x \leq \tau\} \geq r$,
    which proves the second implication.
    
    It follows that $\tau$ is a rank-$r$ threshold if and only if every point $x \in L_r$ satisfies $x \leq \tau$ and every point $x \in U_{|D|-r}$ satisfies $x \geq \tau$.
\end{proof}

\begin{lemma}\label{lem:rankCharacterization}
    Let $D$ be a multiset of real numbers and $\tau \in \reals$ be any threshold. Then
    \begin{align*}
    \ranks(\tau, D)
    = \left\{
       \sum_{x \in D} \ind\{x < \tau\},
        \ldots,
        \sum_{x \in D} \ind\{x \leq \tau\}
    \right\}.
    \end{align*}
\end{lemma}
\begin{proof}
    \newcommand \rmin {r_{\min}}
    \newcommand \rmax {r_{\max}}

    By definition, $\tau$ is a rank-$r$ threshold if $\sum_{x \in D} \ind\{x < \tau\} \leq r$ and $\sum_{x \in D}\ind\{x \leq \tau\} \geq r$.
    Let $\rmin = \sum_{x \in D} \ind \{x < \tau\}$ and $\rmax = \sum_{x \in D} \{x \leq \tau\}$.
    We will show that $\tau$ is a rank-$r$ threshold for $D$ if and only if $\rmin \leq r \leq \rmax$.

    First, since $\rmin = \sum_{x \in D} \ind\{x < \tau\}$, we have that $\sum_{x \in D} \ind\{x < \tau\} \leq r$ if and only if $\rmin \leq $.
    Similarly, since $\rmax = \sum_{x \in D} \ind \{x \leq \tau\}$, we have that $\sum_{x \in D} \ind \{x \leq \tau\} \geq r$ if and only if $r \leq \rmax$.
    Since $\tau$ is a rank-$r$ threshold if and only if both of the above inequalities hold, it follows that $\tau$ is a rank-$r$ threshold iff $\rmin \leq r \leq \rmax$, as required.
\end{proof}

\thmPrivateThreshold*
\begin{proof}
    \Cref{alg:PrivateThreshold} is an instance of the exponential mechanism, so to prove that it satisfies $\eps$-differential privacy, it is sufficient to argue that the sensitivity of the loss $\ell(\tau)$ is bounded by one.
    Let $D$ and $D'$ be any neighboring datasets. %
    Since adding or removing a point from $D$ changes the value of each sum in the expression for $\ranks(\tau', D)$ given by \Cref{lem:rankCharacterization} by at most one, we are guaranteed that whenever $r' \in \ranks(\tau', D)$, then at least one of $r'-1$, $r'$, or $r'+1$ belongs to $\ranks(\tau', D')$. 
    From this, it follows that $|\rankerror(\tau', r, D) - \rankerror(\tau', r, D')| \leq 1$.
    Taking the minimum of both sides of this inequality with respect to $\tau' \in [\tau-\alpha, \tau+\alpha]$ shows that the sensitivity of $\ell$ is bounded by one, as required.

    Next we argue that there exists an interval $I^* \subset [a,b]$ of width at least $\alpha$ so that for every $\tau \in I^*$ we have $\ell(\tau) = 0$.
    Let $\tau^* \in [a,b]$ be any rank-$r$ threshold for the dataset $D$.
    Next, define $I^* = [\tau^* - \alpha, \tau^* + \alpha] \cap [a,b]$.
    The width of $I^*$ is at least $\alpha$ (since at least half of it is contained in $[a,b]$).
    Moreover, for every $\tau \in I^*$ we have that $\ell(\tau) = 0$, since $\tau$ is within distance $\alpha$ of an exact rank-$r$ threshold, as required.
    
    Finally, we follow the standard analysis of the exponential mechanism to prove that this is sufficient to find an approximate rank-$r$ threshold.
    For any $c \geq 0$, define $S_c = \{\tau \in [a,b] \mid \ell(\tau) \geq c\}$ to be the set of thresholds whose loss is at least $c$.
    We have that
    \begin{align*}
    \int_{\tau \in S_c} \exp\left(-\frac{\eps}{2} \ell(\tau) \right)\, d\tau
    &\leq \int_{\tau \in S_c} \exp\left(-\frac{\eps c}{2} \right)\, d\tau \\
    &\leq (b-a) \cdot \exp\left(-\frac{\eps c}{2}\right).
    \end{align*}
    On the other hand, we have
    \[
    \int_a^b \exp\left(- \frac{\eps}{2} \ell(\tau) \right) \, d\tau
    \geq
    \int_{I^*} \exp(0) \, d\tau
    \geq \alpha.
    \]
    Together, it follows that
    \[
    \Pr(\hat \tau \in S_c)
    = \int_{\tau \in S_c} f(\tau) \, d\tau
    \leq \frac{b-a}{\alpha} \exp\left(-\frac{\eps c}{2}\right).
    \]
    Choosing $c = \frac{2}{\eps} \ln \frac{b-a}{\alpha \zeta}$ results in $\Pr(\hat \tau \in S_c) \leq \zeta$.
    
    It follows that with probability at least $1-\zeta$, we have that $\ell(\hat \tau) < \frac{2}{\eps} \ln \frac{b-a}{\alpha\zeta}$.
    From the definition of the loss, it follows that $\hat \tau$ is an $(\alpha,\beta)$-approximate rank-$r$ threshold for $S$.
    
    Finally, we prove the running time guarantee.
    First, the loss function $\ell(\tau)$ is piecewise constant with at most $2|D|$ discontinuities.
    This is because, as we slide a threshold $\tau$ from left to right, the minimum rank error within the interval $[\tau - \alpha, \tau+\alpha]$ only changes when an endpoint of the interval crosses a datapoint in $D$, which can hapen at most $2|D|$ times.
    It follows that the output distribution of \Cref{alg:PrivateThreshold} is also piecewise constant with discontinuities (potentially) occurring at $x \pm \alpha$ for each $x \in D$.
    Let the constant intervals be $I_1, \dots, I_M$ and let $p_1, \dots, p_M$ be the value of the exponential mechanism density on the intervals, respectively.
    Now let $\hat \tau$ be a sample from the output distribution and $\hat I \in \{I_1, \dots, I_M\}$ be the interval that contains $\hat \tau$.
    The key idea behind the sampling strategy is to sample $\hat I$ first, and then sample $\hat \tau$ conditioned on the choice of $\hat I$.
    Since the density is constant on each interval $I_1, \dots, I_M$, the second step is equivalent to outputting a uniformly random sample from $\hat I$.
    This works as long as the probability we choose $\hat I = I_i$ is equal to the marginal distribution of $\hat I$, which is given by $\prob(\hat I = I_i) \propto \operatorname{width}(I_i) \cdot p_i$.
    
    The running time of the above approach is dominated by the cost of computing the piecewise constant  representation of the output density.
    This can be accomplished by constructing the set of $2|D|$ candidate discontinuity locations $x \pm \alpha$ for $x \in D$, sorting them, and then making a linear pass from left to right computing the constant intervals and the value of $\ell(\tau)$ on each interval.
    The overall running time of this is $O(|D| \log |D|)$.
\end{proof}

\corDiscreteThreshold*
\begin{proof}
    For any threshold $\tau$, \Cref{lem:rankCharacterization} guarantees that
    \[
    \ranks(\tau, D) = \left\{
    \sum_{x \in D} \ind\{x < \tau\}
    ,\ldots,
    \sum_{x \in D} \ind\{x \leq \tau\}
    \right\}.
    \]
    When the dataset $D$ is supported on a grid $\cZ$, moving $\tau$ to its nearest grid point never removes ranks from $\ranks(\tau, D)$, since the left sum is either the same or decreases, and the right sum is either the same or increases.
    This implies that moving $\tau$ to its closest grid point never increases its rank error.

    By \Cref{thm:PrivateThreshold} we are guaranteed that with probability at least $1-\zeta$, there exists $\tau'$ with $|\hat \tau - \tau'| \leq \alpha$ and $\rankerror(\tau', r, D) \leq \frac{2}{\eps} \ln \frac{b-a}{\alpha \zeta} \leq \beta$ (where we used the fact that $(b-a)/\alpha \leq 3m$).
    Assume this high probability event holds for the remainder of the proof.
    
    Let $\tilde \tau$ and $\tilde \tau'$ be the closest grid points to $\hat \tau$ and $\tau'$, respectively.
    If $\tilde \tau = \tilde \tau'$ then we have that $\rankerror(\tilde \tau, r, D) = \rankerror(\tilde \tau', r, D) \leq \rankerror(\tau', r, D) \leq \beta$ and we have shown that $\tilde \tau$ is a $(0,\beta)$-approximate rank-$r$ threshold for $D$.
    
    Now suppose that $\tilde \tau \neq \tilde \tau'$.
    First we argue that $D \cap [\hat \tau,  \tau']$ must be the empty set.
    Suppose for contradiction that there is $x \in D \cap [\hat \tau, \tau ']$.
    Then, $x$ is a grid point, and we have that $\max\{|\hat \tau - x|, |\tau' - x|\} \leq |\hat \tau - \tau'| \leq \alpha = \gamma/3$.
    In particular, this implies that both $\hat \tau$ and $\tau'$ are within distance $\gamma/3$ of the same grid point, which means we must have $\tilde \tau = \tilde \tau'$, which is a contradiction.
    Since there are no datapoints in $[\hat \tau, \tau']$, by \Cref{lem:rankCharacterization} we have that $\ranks(\hat \tau, D) = \ranks(\tau', D)$. 
    It follows that
    $\rankerror(\tilde \tau, r, D) \leq \rankerror(\hat \tau, r, D) = \rankerror(\tau', r, D) \leq \beta$.

    In either case, we showed that the rank-error of $\tilde \tau$ is bounded by $\beta$.
\end{proof}
\section{Proofs of mean estimation}

\subsection{Proof of Lemma~\ref{thm:MeanLowerBound}}
\label{app:MeanLowerBound}
\znew{\textbf{Upper bound. $c_{D,\eps} \le 1$.} This can be seen since for all $D' \subset D', |D'| \ge |D| - 1/\eps$, we have
\[
    \mu(D\setminus L_{\frac{1}{\eps}}) \le \mu(D') \le   \mu(D\setminus U_{\frac{1}{\eps}}).
\]
Hence a fixed algorithm that outputs $\mu(D)$ will always have
\[
    |\mu(D) - \mu(D')| \le  \mu(D\setminus L_{\frac{1}{\eps}}) - \mu(D\setminus U_{\frac{1}{\eps}}).
\]}
\textbf{Lower bound. $c_{D,\eps} \ge 1/(2e^2)$.} The proof follows from substituting $D_1 = D \setminus L_{\frac{1}{\eps}}$ and $D_2 = D \setminus U_{\frac{1}{\eps}}$ in  Lemma~\ref{lem:lower}.

\subsection{Proof of Lemma~\ref{thm:AddRemoveMean}}
\label{app:AddRemoveMean}

First we prove the privacy guarantee. 
Let $D_1$ and $D_2$ be any pair of datasets and let $D'_1$ and $D'_2$ be the shifted and clipped versions of them for the interval $[l,u]$.
The size of the symmetric difference between $D'_1$ and $D'_2$ cannot be larger than between $D_1$ and $D_2$, so whenever $D_1$ and $D_2$ are neighbors, so are $D'_1$ and $D'_2$.
It follows that we can ignore the shifting and clipping step in the privacy analysis.

Since the add/remove sensitivity of $n$ is 1, step 3 estimates $n$ using the Laplace mechanism with a budget of $\eps/2$. 
The add/remove sensitivity of the sum $s$ of the shifted data is $w/2$, so step 4 estimates $s$ using the Laplace mechanism with a budget of $\eps/2$.
The overall privacy guarantee of the algorithm then follows from basic composition and post-processing, since $w$ and $m$ are public quantities (i.e., they depend on the algorithm parameters, not on the actual dataset).

Now we turn to the utility analysis. Recall that $n = |D'| = |D|$. Since $D' = D-m$,
\[
\hat{\mu} - \mu(D) = \clip\left(\frac{\hat s}{\hat n}, \left[-\frac{w}{2},\frac{w}{2}\right]\right) - \mu(D').
\]
Since all all elements of $D'$ lie in $[-\frac{w}{2},\frac{w}{2}]$,
\[
\left \lvert \clip\left(\frac{\hat s}{\hat n}, \left[-\frac{w}{2},\frac{w}{2}\right]\right) - \mu(D') \right \rvert
\leq \left \lvert \frac{\hat s}{\hat n} - \mu(D') \right \rvert 
=  \left \lvert \frac{\hat s}{\hat n} - \frac{s}{n} \right \rvert.
\]
Hence, we can bound the desired expectation as
\begin{align}
\EE \left[ |\hat{\mu} - \mu(D)| \right]
& = \Pr(Z_n < -n/2) \EE \left[ |\hat{\mu} - \mu(D)| | Z_n < -n/2 \right]
+ \Pr(Z_n \geq -n/2) \EE \left[ |\hat{\mu} - \mu(D)| | Z_n \geq -n/2 \right] \nonumber \\
& = \Pr(Z_n < -n/2) \EE \left[ |\hat{\mu} - \mu(D)| | Z_n < -n/2 \right]
+ \EE \left[ \left \lvert \frac{\hat s}{\hat n} - \frac{s}{n} \right \rvert | Z_n \geq -n/2 \right] \nonumber \\
& \leq \frac{1}{2} e^{-n\eps/4} | l - u|
+ \EE \left[ \left \lvert \frac{\hat s}{\hat n} - \frac{s}{n} \right \rvert | Z_n \geq -n/2 \right], \label{eq:clip_mean1}
\end{align}
where the last inequality follows by observing that 
both $\hat{\mu}$ and $\mu(D)$ lie in $[l, u]$. We now simplify the second term in~\eqref{eq:clip_mean1}.
\begin{align}
\EE \left[ \left \lvert \frac{\hat s}{\hat n} - \frac{s}{n} \right \rvert | Z_n \geq -n/2 \right] 
& = \EE \left[ \left \lvert \frac{s + Z_s}{n + Z_n} - \frac{s}{n} \right \rvert | Z_n \geq -n/2 \right]   \nonumber \\
& = \EE \left[ \left \lvert \frac{nZ_s}{(n + Z_n)n} \right \rvert | Z_n \geq -n/2 \right] \nonumber   \\
& = \EE \left[ \left \lvert \frac{Z_s}{(n + Z_n)} \right \rvert | Z_n \geq -n/2 \right]  \nonumber  \\
& \stackrel{(a)}{=} \EE[|Z_s|] \cdot \EE \left[ \left \lvert \frac{1}{(n + Z_n)} \right \rvert | Z_n \geq -n/2 \right]  \nonumber  \\
& \leq \EE[|Z_s|] \cdot \frac{2}{n} \nonumber  \\
& = \frac{(u-l)}{\eps}  \cdot \frac{2}{n},  \label{eq:clip_mean2}
\end{align}
where $(a)$ uses the fact that $Z_s$ and $Z_n$ are independent of each other. Combining~\eqref{eq:clip_mean1} and~\eqref{eq:clip_mean2} yields,
\[
\EE \left[ |\hat{\mu} - \mu(D)| \right]
\leq \frac{1}{2} e^{-n\eps/4} | l - u| +  \frac{2(u-l)}{n\eps} \leq \frac{(u-l)}{n\eps} \left( \frac{2}{e} + 2 \right) \leq  \frac{3(u-l)}{n\eps},
\]
where the penultimate inequality uses the fact that $e^{-x}x \leq e^{-1}$ for all $x \geq 0$.

\subsection{Proof of Theorem~\ref{thm:SubsetOptimalMeans}}
\label{app:SubsetOptimalMeans}

    Our goal is to show that the expected error of \Cref{alg:SubsetOptimalMean} is not much larger than
    \[
    R(D,\eps) \geq \frac{\mu(D\setminus L_{\frac{1}{\eps}}) - \mu(D\setminus U_{\frac{1}{\eps}})}{2e^2}.
    \]

    With probability at least $1-\zeta$, by \Cref{thm:PrivateThreshold} we are guaranteed that $\hat l$ and $\hat u$ are $(\alpha,\beta)$-approximate rank-$t_l$ and rank-$t_u$ thresholds, respectively for the values of $\alpha$ and $\beta$ defined in \Cref{alg:SubsetOptimalMean}.
    In particular, this implies that there exist $l'$ and $t'_l$ such that $|\hat l - l'| \leq \alpha$, $|t'_l - t_l| \leq \beta$, and $l'$ is a rank-$t'_l$ threshold for $D$.
    Similarly, there exist $u'$ and $t'_u$ such that $|\hat u - u'| \leq \alpha$, $|t'_u - t_u| \leq \beta$, and $u'$ is a rank-$t'_u$ threshold for $D$.
    Let $G$ denote this high probability event.
    We first argue that conditioned on $G$, the expected loss of $\hat \mu$ is small (where the expectation is taken only over the randomness of \Cref{alg:AddRemoveMean}).
    To convert this bound into a bound that holds in expectation, we bound the error when $G$ does not hold by $2R$.

    Let $\hat \mu$ be the output of \Cref{alg:SubsetOptimalMean}.
    We decompose the error of $\hat \mu$ into three terms:
    \begin{align*}
        |\hat \mu - \mu(D)|
        &\leq   |\hat \mu - \mu(\clip(D, [\hat l, \hat u]))|\\
        &+ |\mu(\clip(D, [\hat l, \hat u]) - \mu(\clip(D, [l', u'])|\\
        &+ |\mu(\clip(D, [l', u'])) - \mu(D)|
    \end{align*}
    Roughly speaking, the first term captures the variance incurred by using \Cref{alg:AddRemoveMean} to estimate the mean of the clipped data, the second term measures our bias due to $\alpha$ in our $(\alpha,\beta)$-approximate thresholds, and the third term measures the bias due to $\beta$.
    Our goal is to prove that all of these terms are not much larger than $R(D,\eps)$.
    
    \textit{Bounding first term.}
    At a high level, we argue that all points in $L_{\frac{1}{\eps}}$ are to the left of $\hat l + \alpha$, and all points in $R_{\frac{1}{\eps}}$ are to the right of $\hat u - \alpha$.
    It follows that the distance from any point in $L_{\frac{1}{\eps}}$ to any point in $U_{\frac{1}{\eps}}$ is at least $\hat u - \hat l - 2\alpha$.
    In particular, this guarantees that the difference between the means of $D\setminus L_{\frac{1}{\eps}}$ and $D\setminus U_{\frac{1}{\eps}}$ must be at least $\frac{\hat u - \hat l - 2\alpha}{\eps(|D|-\frac{1}{\eps})}$, since we move $1/\eps$ points a distance at least $\hat u - \hat l - 2\alpha$.
    This expression is close to the loss incurred by \Cref{alg:AddRemoveMean} when run on the clipped dataset.
    
    Formally, let $S = (D \setminus L_{\frac{1}{\eps}}) \cap (D \setminus U_{\frac{1}{\eps}})$ be the set of common points in the two means from the lower bound.
    Then we have that
    \begin{align*}
        \mu(D \setminus L_{\frac{1}{\eps}})
        -
        \mu(D \setminus U_{\frac{1}{\eps}})
        &\quad= \frac{1}{|D|-\frac{1}{\eps}} \left(
        \sum_{x \in S} x + \sum_{x \in U_{\frac{1}{\eps}}} x
        - \sum_{x \in S} x
        - \sum_{x \in L_{\frac{1}{\eps}}} x
        \right)\\
        &\quad= \frac{1}{|D|-\frac{1}{\eps}} \left(
        \sum_{x \in U_{\frac{1}{\eps}}} x
        - \sum_{x \in L_{\frac{1}{\eps}}} x
        \right)
    \end{align*}
    Next, since $l'$ is a rank-$t'_l$ threshold with $t'_l \geq \frac{1}{\eps}$, we know that every $x \in L_{t'_l} \supset L_{\frac{1}{\eps}}$ satisfies $x \leq l' \leq \hat l + \alpha$.
    Similarly, since $u'$ is a rank-$t'_u$ threshold with $t'_u \leq |D| + \beta = |D| - \frac{1}{\eps}$, we are guaranteed that every $x \in U_{|D| - t'_u} \supset U_{|D| - |D| + \frac{1}{\eps}} = U_{\frac{1}{\eps}}$ satisfies $x \geq u' \geq \hat u - \alpha$.
    Substituting these bounds into the above expression gives
    \begin{align*}
        2e^2R(D,\eps) 
        &\geq \mu(D \setminus L_{\frac{1}{\eps}})
        -
        \mu(D \setminus R_{\frac{1}{\eps}}) \\
        &\geq
        \frac{\hat u - \hat l - 2\alpha}{\eps|D| - 1} \\
        &\geq
        \frac{\hat u - \hat l}{\eps |D|} - \frac{2\alpha}{\eps |D|}.
    \end{align*}
    By \Cref{thm:AddRemoveMean}, we have that the expectation of the first term in the error decomposition conditioned on the choice of $\hat l$ and $\hat u$ is bounded by $\frac{3(\hat u - \hat l)}{|D|\eps}$.
    It follows that 
    \[
    \expectation{|\hat \mu - \mu(\clip(D, [\hat l, \hat u]))| \,\bigg|\,   G}
    \leq 6e^2 R(D, \eps) + \frac{6\alpha}{\eps |D|}.
    \]

    \textit{Bounding the second term.}
    The key idea behind bounding the second term is that, whenever $\hat l$ is close to $l'$ and $\hat u$ is close to $u'$, then clipping a point $x$ to $[\hat l, \hat u]$ is approximately the same as clipping it to $[l', u']$.
    Formally, we have
    \begin{align*}
    |\mu(\clip(D, [\hat l, \hat u])) - \mu(\clip(D, [l', u']))|
    &\quad\leq
    \frac{1}{|D|} \sum_{x \in D} |\clip(x, [\hat l, \hat u]) - \clip(x, [l', u'])|\\
     &\quad=
    \frac{1}{|D|} \sum_{x \in D} |\min(\hat u, \max(x, \hat l)) - \min(u', \max(x, l'))|\\
    &\quad\leq \frac{1}{|D|} \sum_{x \in D} 2\alpha\\
    &\quad= 2\alpha
    \end{align*}

    \textit{Bounding the third term.}
    Our bound for the third term is the most involved.
    At a high level, we show that the bias introduced by clipping $D$ to the interval $[l', u']$ is at most the worst ``one-sided'' clipping bias incurred clipping the points to the left of $l'$ or to the right of $u'$.
    To see this, observe that when we clip from both sides, the left and right biases cancel out.
    Next, we argue that clipping points to the left of $l'$ (or to the right of $u'$) introduces less bias than \emph{removing} those points.
    This step bridges the gap between \Cref{alg:AddRemoveMean} which clips points and the lower bound on $R(D,\eps)$, which removes points.
    We argue that the number of points removed to the left of $l'$ or to the right of $u'$ is not much larger than $\frac{1}{\eps}$ and use \Cref{lem:RemoveExtraBound} to show that the resulting bias is not much larger than if we had removed exactly $\frac{1}{\eps}$ points instead.
    Finally, to finish the bound, combine our two ``one-sided'' bias bounds to show that the overall bias is never much larger than $R(D,\eps)$.
    
    We begin by showing that the bias is bounded by the worst ``one-sided'' bias.
    We have that
    \[
    \mu(D) = \frac{1}{|D|}\left(
        \sum_{\overset{x \in D}{x < l'}} x
        +
        \sum_{\overset{x \in D}{l' \leq x \leq u'}} x
        +
        \sum_{\overset{x \in D}{x > u'}} x
    \right)
    \]
    and
    \[
    \mu(\clip(D, [l', r'])) = \frac{1}{|D|}\left(
        \sum_{\overset{x \in D}{x < l'}} l'
        +
        \sum_{\overset{x \in D}{l' \leq x \leq u'}} x
        +
        \sum_{\overset{x \in D}{x > u'}} u'
    \right).
    \]
    Therefore, we have that
    \begin{align*}
        |\mu(\clip(D, [l', u'])) - \mu(D)|
        &\quad= \frac{1}{|D|}\left|
            \sum_{\overset{x \in D}{x<l'}} l' - x
            +
            \sum_{\overset{x \in D}{x > u'}} u' - x
        \right|\\
        &\quad \leq \frac{1}{|D|} \max\left\{
            \sum_{\overset{x \in D}{x < l'}} l' - x
            ,
            \sum_{\overset{x \in D}{x > u'}} x - u'
        \right\},
    \end{align*}
    where the inequality follows because the two sums have opposite signs.
    This expression is the maximum bias we introduce if we only cliped either points to the left of $l'$ or to the right of $u'$.
    
    Next, we relate the bias of clipping the points to the left of $l'$ to the bias of removing those points instead.
    Let $q = \sum_{x \in D} \ind\{x < l'\}$ be the number of points that are clipped to $l'$.
    Next, since adding copies of $\mu(D \setminus L_q)$ to $D \setminus L_q$ does not change its mean, we have that
    \begin{align*}
    \mu(D \setminus L_q) - \mu(D) &= \frac{1}{|D|}\left(
    \sum_{x \in D \setminus L_q} (x-x) 
    +
    \sum_{x \in L_q} \mu(D \setminus L_q) - x
    \right)\\
    &= \frac{1}{|D|} \sum_{x \in L_q} \mu(D\setminus L_q) - x\\
    &\geq \frac{1}{|D|} \sum_{x \in L_q} l' - x,
    \end{align*}
    where the final inequality follows from the fact that $\mu(D \setminus L_q) \geq l'$, since every element of $D \setminus L_q$ is at least $l'$.
    We have shown that the bias from clipping the points to $l'$ is at most the bias from deleting them.
    
    Next, we use the fact that  $l'$ is a rank-$t'_l$ threshold for $D$ to argue that the number of points clipped, $q$, cannot be too large.
    Since $l'$ is a rank-$t'_l$ threshold for $D$, we know that every $x \in U_{|D|-t'_l}$ satisfies $x \geq l'$.
    Therefore, 
    \[
    \sum_{x \in D} \ind \{x \geq \tau'_l\}
    \geq |U_{|D|-t'_l} = |D| - t'_l.
    \]
    Since $q = |D| - \sum_{x \in D} \ind\{x \geq \tau'\}$, we have that $q \leq t'_l \leq \frac{1}{\eps} + 2\beta$.
    Putting these bounds together and using \Cref{lem:RemoveExtraBound} to handle the fact that $q$ may be larger than $\frac{1}{\eps}$, we have
    \begin{align*}
        \frac{1}{|D|} \sum_{x \in L_q} l' - x
        &\leq \mu(D \setminus L_q) - \mu(D)\\
        &\leq \frac{2 q}{1/\eps} \bigl(\mu(D \setminus L_{\frac{1}{\eps}}) - \mu(D)\bigr) \\
        &\leq (2+4\beta\eps) \bigl(\mu(D \setminus L_{\frac{1}{\eps}}) - \mu(D)\bigr)
    \end{align*}
    
    Next we turn to the bias of clipping points to the right of $u'$.
    Let $p = \sum_{x \in D} \ind\{x > u'\}$ be the number of points in $D$ that are strictly greater than $u'$.
    A similar argument to the above shows that $p \leq \frac{1}{\eps} + 2\beta$ and that
    \begin{align*}
    \frac{1}{|D|} \sum_{x \in U_q} x - u'
    &\leq \mu(D) - \mu(D \setminus U_p) \leq (2+4\beta\eps)\bigl(\mu(D) - \mu(D\setminus R_{\frac{1}{\eps}})\bigr).
    \end{align*}
    
    Putting it all together, the third term of the error decomposition is upper bounded by
    \begin{align*}
    &|\mu(\clip(D,[l',u'])) - \mu(D)|\\
    &\leq (2+4\beta\eps)\max\{
    \mu(D \setminus L_{\frac{1}{\eps}})
    - \mu(D)
    ,
    \mu(D) - \mu(D \setminus R_{\frac{1}{\eps}})\\
    &\leq (2+4\beta\eps)\bigl(\mu(D \setminus L_{\frac{1}{\eps}}) - \mu(D \setminus U_{\frac{1}{\eps}})\bigr)\\
    &\leq 2e(2+4\beta\eps)R(D,\eps),
    \end{align*}
    where the second inequality follows from the fact that the maximum of two numbers is not larger than the sum.
    
    Finally, conditioned on the good event $G$, we have shown that the expected loss of $\hat \mu$ is bounded by
    \[
    \bigl(2e(2+4\beta\eps^2) + 6e^2\bigr) R(D, \eps) + \alpha \left(\frac{6}{\eps |D|} + 2\right).
    \]
    Using the fact that $\beta = \frac{2}{\eps} \ln \frac{2R}{\alpha\zeta}$ and bounding the error when the good event $G$ fails to hold by $2R$, we have that
    
    \begin{align*}
    &\expectation{|\hat \mu - \mu(D)|}
    \\
    &\leq \left(56 + 44 \ln \frac{2R}{\alpha\zeta}\right)R(D, \eps) + \alpha \left( \frac{6}{\eps |D|} + 1\right) + 2R\zeta,
    \end{align*}
    as required.

\begin{lemma}\label{lem:RemoveExtraBound}
    Let $D$ be any multiset of real numbers of size $n$, and let $n_1 \leq n_2 \leq n/2$.
    Then we have
    \[
    \mu(D \setminus L_{n_2}) - \mu(D) 
    \leq \frac{2n_2}{n_1} \bigl(\mu(D \setminus L_{n_1}) - \mu(D)\bigr)
    \]
    and
    \[
    \mu(D) - \mu(D \setminus R_{n_2}) 
    \leq \frac{2n_2}{n_1} \bigl(\mu(D) - \mu(D \setminus R_{n_1})\bigr)
    \]
\end{lemma}

\begin{proof}
We only prove the first inequality and the second will follow similarly. For any $n' \le n/2$, we have
\begin{align*}
     \absv{\mu_{D} - \mu_{D\setminus L_{n'}}} & = \frac1\ns \sum_{i \in [\ns]} (\mu_{D\setminus L_{n'}} - X_i) \mathbf{1}\{ X_i \in L_{n'}\} \\
     & =  \sum_{i \in [\ns]} (\mu_{D\setminus L_{n'}} - X_i) \mathbf{1}\{ X_i < a\} + (n - n') \mu_{D\setminus L_{n'}} -  (n - n') \mu_{D\setminus L_{n'}}  \\
     & =  n'\mu_{D\setminus L_{n'}} + (n - n') \mu_{D\setminus L_{n'}}  - \sum_{i \in [\ns]} \mathbf{1}\{ X_i < a\} X_i -  \sum_{i \in [\ns]} \mathbf{1}\{ X_i \ge a\} X_i \\
     & =  \ns \Paren{\mu_{D\setminus L_{n'}} - \mu_D}.
\end{align*}

Setting $n'$ to be $n_1$, $n_2$ respectively, it would be enough to prove that
\[
    \mu_{D\setminus L_{n_2}} - \mu_{L_{n_2}}  \le 2 \Paren{\mu_{D\setminus L_{n_1}} - \mu_{L_{n_1}}}.
\]
This follows since
\begin{align*}
   2 \Paren{\mu_{D\setminus L_{n_1}} - \mu_{L_{n_1}}} -  \Paren{\mu_{D\setminus L_{n_2}} - \mu_{L_{n_2}} }  & \ge   2 \mu_{D\setminus L_{n_1}} -  \mu_{D\setminus L_{n_2}} -   \mu_{L_{n_1}} \\
   & = 2 \mu_{D\setminus L_{n_1}} -  \frac{(n - n_1) \mu_{D\setminus L_{n_1}} - \sum_{i = n_1 + 1}^{n_2} X_i}{n - n_2}-   \mu_{L_{n_1}}\\
   & = \Paren{2 - \frac{n - n_1}{n - n_2}} \mu_{D\setminus L_{n_1}} + \frac{n_2 - n_1}{n - n_2}  \mu_{L_{n_2}\setminus L_{n_1}} - \mu_{L_{n_1}} \\
   & \ge 0.
\end{align*}
\end{proof}
\section{Inverse sensitivity mechanism} \label{app:inverse}

In this section, we state the inverse sensitivity mechanism and its guarantee in the add/remove model of DP. Most of the proof follows from \cite{asi2020instance} in the replacement model of DP. We include its add/remove variant here for completeness.

We first provide a few definitions. Let $D$ be a dataset supported over $\cZ$ and $\forall k \in \NN_+$, let $\omega_\ell(D, k)$ be defined as
\[
    \omega_{\ell}(D, k) \eqdef \max \left\{ \ell(\theta(D), \theta(D')) \mid D' \in \cZ^\ns, d(D, D') \le k\right\}
\]
\ie the maximum change in the parameter by $k$ add/remove operations on $D$. 
Let 
\[
{\rm len}_\theta(D, t) \eqdef \min \left\{d(D', D)\mid \theta(D') = t \right\},
\]
which is the minimum number of add/remove operations needed to change $D$ to some $D'$ with $\theta(D') = t$. Then the add/remove version of inverse sensitivity mechanism is stated below.

\begin{algorithm*}
\textbf{Input:} Range $R$, dataset $D \subset [-R, R]$, privacy parameter $\eps > 0$ and granularity $\beta > 0$.
\begin{algorithmic} [1]
    \STATE Let $\cT$ be the set of points in $[-R, R]$ that are also mulitples of $\beta$. 
    \STATE Output $t \in \cT$ with distribution 
\[
    \probof{A(D) = t} = \frac{\exp\Paren{-{\rm len}_\theta (D, t)}}{\sum_{t' \in \cT } \exp\Paren{-{\rm len}_\theta (D, t')} },
\]
\end{algorithmic}
\caption{Inverse Sensitivity Mechanism \citep{asi2020instance}}
\label{alg:inverse_sensitivity}
\end{algorithm*}

The following guarantee holds for the inverse sensitivity mechanism.
\begin{theorem}[\cite{asi2020instance}]
For any $\beta \in (0, B)$, the inverse sensitivity mechanism $A$ with privacy parameter $\eps' = 2 \eps \log \frac{2B R}{\beta}$ satisfies that
\[
    \expectation{\ell\paren{A(D), \theta(D)}} \le \omega_\ell(D, 1/\eps) + L \beta.
\]
\end{theorem}
\section{Proofs for the general algorithm}
\subsection{Proof of the Lemma~\ref{lem:lower}}
\label{app:lower}

For any algorithm $A$,
\begin{align*}
 \max_{D' \subseteq D : |D'| \geq |D| - 1/\eps} \EE\left[\ell\Paren{A(D'), \theta(D')}\right] 
      & \geq  \max_{(D_1, D_2) \in S}   \max \left( \EE\left[\ell\Paren{A(D_1), \theta(D_1)}\right] , \EE\left[\ell\Paren{A(D_2), \theta(D_2)}\right]\right)  \\
       & \geq \max_{(D_1, D_2) \in S}   0.5 \cdot \left( \EE\left[\ell\Paren{A(D_1), \theta(D_1)}\right] + \EE\left[\ell\Paren{A(D_2), \theta(D_2)}\right]\right).
\end{align*}
Since $d(D_1, D_2) \leq 2 /\epsilon$, if $A \in A_\epsilon$, 
\begin{align*}
 \EE\left[\ell\Paren{A(D_1), \theta(D_1)}\right] + \EE\left[\ell\Paren{A(D_2), \theta(D_2)}\right] 
& \geq e^{-(2/\eps)\eps} \EE\left[\ell\Paren{A(D_2), \theta(D_1)}\right] + \EE\left[\ell\Paren{A(D_2), \theta(D_2)}\right] \\
& \geq e^{-2} \left( \EE\left[\ell\Paren{A(D_2), \theta(D_1)}\right] + \EE\left[\ell\Paren{A(D_2), \theta(D_2)}\right] \right)\\
& \geq e^{-2} \ell\Paren{\theta(D_1), \theta(D_2)}\\
& \geq e^{-2} \ell\Paren{\theta(D_1), \theta(D_2)}.
\end{align*}
Hence, for any algorithm $A \in A_\epsilon$,
\begin{align*}
 \max_{D' \subseteq D : |D'| \geq |D| - 1/\eps} \EE\left[\ell\Paren{A(D'), \theta(D')}\right] 
 \geq 
 \max_{\max_{(D_1, D_2) \in S}}   \frac{1}{2e^2} \ell\Paren{\theta(D_1), \theta(D_2)}.
\end{align*}

\subsection{Proof of Theorem~\ref{thm:general}}
\label{app:general}

We first prove a general result on stochastic dominance which will be helpful later in our results. For a dataset $D$ and thresholds $l, u$, 
let $D_{[l, u]} = D \cap [l, u]$.

\begin{lemma} \label{lem:domination}
Let $l$ and $u$ satisfy, 
\begin{align*}
2r \geq |D \cap (-\infty, l)| \geq
\frac{3}{2} r, \\
2r \geq |D \cap (u, \infty)| \geq \frac{3}{2} r. 
\end{align*}
For all $D'$ such that all of their elements are in  $[l, u]$ and $d(D_{[l, u]}, D') \le r/2$, 
\[
D\setminus L_{4r}(D) \succ D' \succ D\setminus U_{4r} (D).
\]
\end{lemma}
\begin{proof}
Let $L_{4r}$ denote $L_{4r}(D)$ and $U_{4r}$ denote $U_{4r}(D)$ for convenience. We present the proof for $D\setminus L_{4r} \succ D'$ and the other relation will follow similarly. We divide the proof into three cases depending on the value of $v$ in Definition~\ref{def:first_order}.

\noindent \textbf{Case $v < \min_{x \in  D\setminus L_{4r}} x$:} Since there are no points in $D\setminus L_{4r}$ in this range, %
\[
     \frac{\sum_{x \in D\setminus L_{4r}} \indic{x \leq v}}{|D\setminus L_{4r}|} = 0 \leq \frac{\sum_{x \in D'} \indic{x \leq  v}}{|D'|}.
\]
\noindent \textbf{Case $v \geq u$:} Since all points of $D'$ lie below $u$,
\[
   \frac{ \sum_{x \in D\setminus L_{4r}} \indic{x \leq v}}{|D\setminus L_{4r}|} \leq 1 =  \frac{\sum_{x \in D'} \indic{x \leq v}}{|D'|} 
\]
\noindent \textbf{Case $v \in [\min_{x \in  D\setminus L_{4r}} x, u)$:} Since $d(D_{[l, h]}, D') \leq r/2$,
\begin{align*}
\frac{\sum_{x \in D'} \indic{x \leq v}}{|D'|} 
& \geq \frac{\sum_{x \in D_{[l, u]}} \indic{x \leq v} - r/2}{|D_{[l, u]}| + r/2} \\
& \geq \frac{\sum_{x \in D_{[l, u]}} \indic{x \leq v} - r/2}{|D\setminus L_{4r}| + 3r/2},
\end{align*}
where the last inequality follows by observing that the assumptions in the lemma imply,
\[
|D_{[l, u]}| + r/2\leq |D| - 5r/2 \leq |D\setminus L_{4r}|+ 3r/2.
\]
For $v \in [\min_{x \in  D\setminus L_{4r}} x, u)$,
\begin{align*}
\sum_{x \in D_{[l, u]}} \indic{x \leq v} - r/2
& \geq \sum_{x \in D} \indic{x \leq v} - 2r - r/2 \\
& =\sum_{x \in D} \indic{x \leq v} - 5r/2 \\
& =\sum_{x \in D \setminus  L_{4r}} \indic{x \leq v} + 4r - 5r/2 \\
& = \sum_{x \in D \setminus  L_{4r}} \indic{x \leq v} + 3r/2.
\end{align*}
Hence,
\begin{align*}
\frac{\sum_{x \in D'} \indic{x \leq v}}{|D'|} 
& \geq \frac{\sum_{x \in D \setminus  L_{4r}} \indic{x \leq v} + 3r/2}{|D\setminus L_{4r}|+ 3r/2} \\
& \geq  \frac{\sum_{x \in D \setminus  L_{4r}} \indic{x \leq v} }{|D\setminus L_{4r}|}.
\end{align*}
\end{proof}

\begin{proof}[Proof of Theorem~\ref{thm:general}]
\new{In this proof, we show that \cref{alg:priv_est} is $\eps$-DP and achieves 
\begin{align*}
 \EE[\ell(A(D),\theta(D))] 
 \leq 2e^2 R(D, \eps') + 7 L\beta,
\end{align*}
where $\eps' = \frac{\eps}{128\log(6RB/ L\beta^2)}$. \cref{thm:general} can be obtained by applying \cref{alg:priv_est} with privacy parameter $128\log(6RB/ L\beta^2)\eps$.
} 

By composition theorem, the overall privacy budget is $\eps$. In the rest of the proof, we focus on the utility guarantee. We first quantify the effect of quantization. Observe that the output of the algorithm does not change for inputs $D$ and the corresponding quantized dataset $D_\text{quant}$. Hence together with the Lipschitz property of $\theta$,
\begin{align*}
 \EE [\ell(A(D), \theta(D)] 
& = \EE [\ell(A(D_\text{quant}), \theta(D)] \\
& \leq \EE [\ell(A(D_\text{quant}), \theta(D_\text{quant})] + \ell(\theta(D_\text{quant}), \theta(D)),  \\
& \leq  \EE [\ell(A(D_\text{quant}), \theta(D_\text{quant})] + L\beta. 
\end{align*}
By Corollary~\ref{cor:threshold}, with probability at least $1-\eta$, there exists a $r'$ such that $|r'-7r/4| \leq \frac{8}{\eps} \log \frac{6R}{\eta\beta}$ and $r' \in R(l, D_{\text{quant}})$. In other words,
\[
R(l, D_{\text{quant}}) \cap [3r/2, 2r] \neq \emptyset.
\]
Hence,
\[
|D_{\text{quant}} \cap (-\infty, l)| \leq 2r,
\]
and hence,
\[
3r/2 \leq |D'_{\text{quant}} \cap (-\infty, l)| \leq 2r.
\]
Similarly,
\[
3r/2\leq |D'_{\text{quant}} \cap (u, \infty)| \leq 2r.
\]
Let $E$ be the event where both the above equations hold.  By triangle inequality,
\begin{align*}
 \EE [\ell(A(D'_\text{quant}), \theta(D'_\text{quant})]  \leq \EE [\ell(A(D'_\text{quant}), \theta(D_{[l, h]})] + \EE [\ell(\theta(D'_\text{quant}), \theta(D_{[l, h]})],.
\end{align*}
Let $L'_{4r} = L_{4r}(D'_\text{quant})$ and $U'_{4r} = U_{4r}(D'_\text{quant})$. By Lemma~\ref{lem:domination}, $ D'_\text{quant}\setminus L_{4r} \succ D_{[l, u]} \succ D'_\text{quant}\setminus U_{4r}$. Hence, conditioned on the event $E$, by the monotonicity property
\begin{align*}
 \EE [\ell(\theta(D'_\text{quant}), \theta(D_{[l, h]})] 
& \leq \max \left(\ell(\theta(D'_\text{quant}), \theta(D'_\text{quant}\setminus L'_{4r}), \ell(\theta(D'_\text{quant}), \theta(D'_\text{quant}\setminus U'_{4r}) \right).
\end{align*}
Furthermore by triangle inequality,
\begin{align*}
 \ell(\theta(D_\text{quant}), \theta(D'_\text{quant}\setminus L'_{4r}) 
& \leq \ell(\theta(D), \theta(D\setminus L_{4r}) + \ell(\theta(D), D'
_\text{quant}))
+ \ell(\theta(D\setminus L_{4r}, \theta(D'_\text{quant}\setminus L'_{4r}) \\
& \leq\ell(\theta(D), \theta(D\setminus L_{4r})+ 4L \beta.
\end{align*}
Similarly, 
\[
\ell(\theta(D'_\text{quant}), \theta(D'_\text{quant}\setminus U'_{4r})
\leq \ell(\theta(D), \theta(D\setminus U_{4r})+ 4L \beta.
\]
Hence, 
\begin{align*}
\ell(\theta(D'_\text{quant}), \theta(D_{[l, h]}) 
& \leq \max \left(\ell(\theta(D), \theta(D\setminus L'_{4r}), \ell(\theta(D), \theta(D\setminus U'_{4r}) \right) + 4 L \beta \\
& \leq \max_{D_1, D_2 \subseteq D: d(D_1, D_2) \leq 4r} \ell(\theta(D_1), \theta(D_2))+ 4 L \beta.
\end{align*}
Therefore,
\begin{align*}
\EE[\ell(\theta(D_\text{quant}), \theta(D_{[l, h]})]  
& \leq \max_{D_1, D_2 \subseteq D: d(D_1, D_2) \leq 4r} \ell(\theta(D_1), \theta(D_2))+ 4 L \beta + \text{Pr}(E) B \\
& = \max_{D_1, D_2 \subseteq D: d(D_1, D_2) \leq 4r} \ell(\theta(D_1), \theta(D_2))+ 6 L \beta. \\
\end{align*}
Since inverse sensitivity mechanism is applied on $D_{[l, u]}$,
\[
\EE [\ell(A(D'_\text{quant}), \theta(D_{[l, h]})] =
\EE [\ell(\text{InvSen}(D_{[l, h]}), \theta(D_{[l, h]})].
\]
By the guarantee of the inverse sensitivity mechanism, 
\begin{align*}
 \EE [\ell(\text{InvSen}(D_{[l, h]}), \theta(D_{[l, h]})] 
& \leq \max_{D_1, D_2 \subseteq D: d(D_1, D_2) \leq r'} \ell(\theta(D_1), \theta(D_2)) + L\beta,
\end{align*}
where $r' = \left(\frac{4\log((8BR)/L\beta^2)} {\eps}\right)$.
Combining the above equations yield
\begin{align*}
\EE[\ell(A(D),\theta(D))] 
& \leq 2\max_{D_1, D_2 \subseteq D: d(D_1, D_2) \leq \max(4r, r')} \ell(\theta(D_1), \theta(D_2)) + 7 L \beta \\
& \leq 2e R(D, \eps') + 7 L\beta,
\end{align*}
where $\eps' = \frac{\eps}{128\log(6RB/ L\beta^2)}$.
\end{proof}

\section{Proofs for statistical mean estimation}
\subsection{Proof of Corollary~\ref{coro:moment}}
\label{app:cor_moment}
    It is straightforward to see that
    \[
        \expectation{\absv{\mu(X^\ns) - \mu}} \le \sqrt{\expectation{\absv{\mu(X^\ns) - \mu}^2}} = \sqrt{\frac{M_2(p)}{\ns}} 
    \]
    Hence it remains to show that $\forall k \ge 2$,
    \begin{align}
        & \;\;\;\; 
        \expectation{\absv{\mu(X^\ns\setminus L_{\frac{1}{\epsilon}}) - \mu(X^\ns\setminus U_{\frac{1}{\epsilon}})}} 
        O\Paren{\sqrt{\frac{M_2(p)}{\ns}}  + \frac{M_k(p)^{1/k}}{(n\eps)^{1-1/k}} } \label{eqn:small_expect_diff}.
    \end{align} 
    
    Our proof will be based on the lemma below.
    \begin{lemma} \label{lem:moment_dataset}
        For any sequence of samples $X^\ns$ and $k \ge 2$, let $\widehat{M}_k(X^\ns) \eqdef \frac{1}{\ns} \sum_{i = 1}^\ns |X_i - \mu(X^\ns)|^k$, we have
        \[
            \absv{\mu(X^\ns\setminus L_{\frac{1}{\epsilon}}) - \mu(X^\ns\setminus U_{\frac{1}{\epsilon}})} = O\Paren{\frac{\widehat{M}_k(X^\ns)^{1/k}}{(n\eps)^{1-1/k}}}
        \]
    \end{lemma}
    We first prove \cref{eqn:small_expect_diff} based on \cref{lem:moment_dataset} and then present the proof of \cref{lem:moment_dataset}.
     \begin{align*}
         \expectation{\absv{\mu(X^\ns\setminus L_{\frac{1}{\epsilon}}) - \mu(X^\ns\setminus U_{\frac{1}{\epsilon}})}} = O\Paren{ \expectation{\frac{\widehat{M}_k(X^\ns)^{1/k}}{(n\eps)^{1-1/k}}}}
     \end{align*}
     Moreover, 
     \begin{align*}
      \expectation{ \widehat{M}_k(X^\ns)^{1/k}} 
        & = \expectation{ \Paren{\frac{\sum_{i = 1}^\ns |X_i - \mu(X^\ns)|^k}{\ns}}^{1/k}} \\
        & \le 2 \expectation{ \Paren{\frac{\sum_{i = 1}^\ns |X_i - \mu(p)|^k}{\ns}}^{1/k} \!\!\!\!\!\!+ \!\!\! \Paren{\frac{\sum_{i = 1}^\ns |\mu(X^\ns) - \mu(p)|^k}{\ns}}^{1/k} } \\
        & = 2 \expectation{ \frac{\sum_{i = 1}^\ns |X_i - \mu(p)|^k}{\ns}}^{1/k} + 2 \expectation{|\mu(X^\ns) - \mu(p)|} \\
        & \le 2 M_k(p)^{1/k} +  \frac{2 M_2(p)^{1/2}}{\sqrt{\ns}}.
     \end{align*}
\begin{proofof}{\cref{lem:moment_dataset}}
By definition,
\begin{align*}
 \absv{\mu(X^\ns\setminus L_{\frac{1}{\epsilon}}) - \mu(X^\ns\setminus U_{\frac{1}{\epsilon}})}  & = \frac1{\ns -1/\eps} \absv{\sum_{i \in L_{\frac{1}{\epsilon}}} X_i - \sum_{i \in U_{\frac{1}{\epsilon}}}X_i} \\
   & \le \frac2\ns  \sum_{i \in L_{\frac{1}{\epsilon}} \cup U_{\frac{1}{\epsilon}}}  \absv{ X_i - \mu(X^\ns)} \\
   & \le \frac{2}{\ns}\Paren{ \sum_{i \in L_{\frac{1}{\epsilon}} \cup U_{\frac{1}{\epsilon}}}  \absv{ X_i - \mu(X^\ns)}^k}^{1/k} \cdot \Paren{ \sum_{i \in L_{\frac{1}{\epsilon}} \cup U_{\frac{1}{\epsilon}}}  1 }^{(k-1)/k} \\
   & \le \frac{2}{\ns}\Paren{ \sum_{i \in [\ns]}  \absv{ X_i - \mu(X^\ns)}^k}^{1/k} \cdot \Paren{ \sum_{i \in L_{\frac{1}{\epsilon}} \cup U_{\frac{1}{\epsilon}}}  1 }^{(k-1)/k}\\
   & = \frac{4\widehat{M}_k(X^\ns)^{1/k}}{(n\eps)^{1-1/k}}.\end{align*}
\end{proofof}
\znew{\section{Comparison between different instance-dependent risks for $\ell_p$ minimization.} \label{app:lp}

Consider the task of estimating the $\ell_p$ minimizer of a dataset over $[0, R]$ with $R \gg 1$, \ie for $p > 1$,
\[
    \theta(D) = \min_{\mu \in \RR} \sum_{x \in D}|x - \mu|^p.
\]
Consider a dataset $D$ consisting of $n-1$ $0$'s and one $1$. It can be shown that the minimizer is 
\[
\mu_p(D) = \frac1{1 + (\ns - 1)^{1/(p-1)}}.
\]

Let $\ell(x, x') = |x - x'|$. The definition in \cite{asi2020instance} (\cref{eqn:asi}) will have
\[
    R_{1} (D, \eps) \approx \max_{D': d(D, D') \leq 1/\eps} \ell(\theta(D), \theta(D')) \ge \frac R{1 + (\ns\eps - 1)^{1/(p-1)}} - \frac1{1 + (\ns - 1)^{1/(p-1)}},
\]
where we take $D_1'$ to be the dataset with $\ns - 1/\eps$ $0$'s and $1/\eps$ $R$'s.

The modified definition of \cite{huang2021instance} (\cref{rmk:huang}) will lead to a instance dependent risk of 
\[
    \tilde{R}_{2} (D, \eps) = \sup_{\substack{\support(D')  \subseteq \support(D) \\ d(D, D') \le 1/\eps}} \ell(\theta(D), \theta(D')) \ge  \frac 1{1 + (\ns\eps - 1)^{1/(p-1)}} - \frac1{1 + (\ns - 1)^{1/(p-1)}},
\]
where we take $D'_2$ be the dataset with $\ns - 1/\eps$ $0$'s and $1/\eps$ $1$'s.

Our propose subset-risk will only  allow $D' \subset D$. Hence $\mu_p(D') \in [0, \mu_p(D)] $, and
\[
    R(D, \eps) \approx \sup_{D'\subset D, d(D, D') \le 1/\eps} \ell(\theta(D), \theta(D')) \le \frac1{1 + (\ns - 1)^{1/(p-1)}}.
\]

In the regimen when $p < \log(n\eps)$,  $\eps \ll 1$ and $R \gg 1$, we have
\[
R_{1} (D, \eps) \gg \tilde{R}_{2} (D, \eps) \gg R(D, \eps).
\]}
\end{document}